\documentclass{article}

\PassOptionsToPackage{numbers,compress}{natbib}
\usepackage[final,main,preprint]{neurips_2026} %

\usepackage{macros}
\usepackage{placeins}

\title{Learn from your own latents and not from tokens:\\ A sample-complexity theory }

\author{%
  Daniel J. Korchinski \thanks{ Equal contribution } \\
  Institute of Physics\\
  EPFL \\
  Lausanne, VD Switzerland \\
  \texttt{daniel.korchinski@epfl.ch} \\
  \And 
  Alessandro Favero \footnotemark[1] \\
  Department of Applied Maths and Theoretical Physics\\
  University of Cambridge \\ 
  Cambridge, UK \\ 
  \texttt{af940@cam.ac.uk} \\
  \And 
  Matthieu Wyart \\ 
  Department of Physics \& Institute of Physics \\
  Johns Hopkins University \& EPFL \\
  Baltimore, MD USA \& Lausanne, VD Switzerland \\
  \texttt{matthieu.wyart@epfl.ch}
}
  
\begin{document}

\maketitle

\begin{abstract}
Generative models, from diffusion models to large language models, achieve remarkable performance but at a cost in training data orders of magnitude larger than what biological learners require. An alternative paradigm has emerged in which networks are trained to predict their \emph{own} latent representations of related views or masked regions, as in data2vec and JEPA -- an idea related to predictive-coding accounts of the cortex. Despite strong empirical results, the theoretical understanding of these methods remains limited. Central questions include: by how much does latent prediction actually improve data efficiency? Is there a benefit to stacking such methods into multi-scale hierarchies? We answer both using as data a tractable probabilistic context-free grammar that captures the compositional structure of natural language and images. Such a grammar generates strings of visible tokens by recursively applying production rules along a tree of hidden symbols of depth $L$. For such data, supervised or token-level SSL require a number of samples \emph{exponential} in $L$ to recover the latent tree; we prove that latent prediction achieves this with a number of samples \emph{constant} in $L$, up to logarithmic factors. We confirm this bound with (i) a hierarchical clustering algorithm, (ii) an end-to-end neural network whose predictor-clusterer modules predict their own latents at each level via gradient descent, and (iii) the first sample-complexity analysis of data2vec, which we show implicitly performs hierarchical latent prediction. This suggests that explicit stacking such as H-JEPA is largely redundant.\looseness=-1 \end{abstract}

\vspace{-1.5em} \section{Introduction} \label{sec:intro}
Generative models have reached striking empirical success. Diffusion models produce realistic images and video~\citep{ho2020denoising, rombach2022high, brooks2024video}, while large language models trained by next-token prediction master grammar, world knowledge and reasoning~\citep{brown2020language, openai2023gpt4, deepseekai2025r1}. Both rest on a single recipe: predict masked or future fragments of the raw signal at massive scale. Yet this success is bought at a cost biological learners do not pay.

Frontier LLMs are trained on $10^{13}$--$10^{14}$ tokens~\citep{hoffmann2022training, touvron2023llama}, more than five orders of magnitude beyond what a child encounters before reaching adult-level competence~\citep{gilkerson2017mapping, linzen2020accelerate, frank2023bridging}; state-of-the-art diffusion models likewise rely on billions of images~\citep{schuhmann2022laion}. This gap suggests that token-level pretraining is far from optimal, and several hypotheses have been advanced to explain it. One is that biological learning is multimodal and grounded~\citep{smith2005development,  alayrac2022flamingo}. A second -- the one we pursue -- is that learning may not be most efficient at the level of raw tokens, but instead in a more abstract latent space~\citep{rao1999predictive, friston2010free, lecun2022path}.
Indeed, predicting semantic concepts or continuous latents has
been shown to improve data efficiency~\citep{tackLLMPretrainingContinuous2025, liuNextConceptPrediction2026}.

An alternative strategy to input reconstruction has gained traction in the last five years: the network is trained to predict
its \emph{own} latent representations of related views, masked regions
or future signals. These targets are themselves
the output of a learned encoder, lifting prediction into increasingly
abstract spaces as learning progresses. The idea relates to computational neuroscience,
where predictive coding posits that the cortex seeks to predict its own future activity~\citep{rao1999predictive,illingLocalPlasticityRules2021,
millidge2022predictive}. A growing family of self-supervised methods instantiates this principle, including  BYOL~\citep{grillBootstrapYourOwn2020},
DINO ~\citep{caronEmergingPropertiesSelfsupervised2021},
data2vec~\citep{baevskiData2vecGeneralFramework2022} and  JEPA
~\citep{lecun2022path}.\looseness=-1

Some advantages of these algorithms over input reconstruction approaches have been discussed in the literature, including the fact that they can be formulated as energy-based methods~\citep{lecun2022path} or are more robust to the presence of irrelevant features in the data~\citep{vanassel2025joint}. Yet understanding why and by how much these methods improve data efficiency remains a fundamental challenge in both machine learning and neuroscience. Another key question is the putative benefits of stacking such algorithms so as to obtain a multi-scale hierarchy of learned representations. This idea, proposed years ago  ~\citep{lecun2022path},  was recently implemented ~\citep{liHiTJEPAHierarchicalSelfsupervised2025, girgisHierarchicalJEPAMeets2026} and  appears to lead at best to moderate improvements, for  reasons that remain currently unclear.\looseness=-1

In this work, we address both questions by focusing on synthetic models of data with a 
hierarchical latent structure. Probabilistic context-free grammars
(PCFGs) are a natural candidate: through recursive production rules
they generate data compositionally from a tree of latent variables, and
have long been influential as models of language ~\citep{booth1973applying, manning1999foundations} and natural images~\citep{grenander1996elements,
mumford1994pattern}. General PCFGs are notoriously hard to
treat. We therefore work with the Random Hierarchy Model
(RHM)~\citep{cagnetta2024deep}, a simplified PCFG
with fixed tree topology and random production rules. The RHM has
produced quantitatively predictive theories, including relating neural scaling laws to the statistics of natural language~\citep{cagnetta2024towards, cagnettaDerivingNeuralScaling2026}, and
a theory of compositional generalization and memorization in diffusion models for
language and images~\citep{sclocchiPhaseTransitionDiffusion2025, sclocchiProbingLatentHierarchical2025,favero2025compositional,favero2025bigger}.
The RHM is parameterized by the number of production rules
per symbol, $m$, and the depth of the latent tree, $L$. Prior work
shows that for deep architectures, supervised end-to-end learning requires an order of $m^L$
samples~\citep{cagnetta2024deep,ren2026provable}, and token-level self-supervised objectives -- masked-language
modeling and diffusion -- require of order
$m^{L+1}$~\citep{cagnetta2024towards,favero2025compositional,cagnetta2025scaling} (cf. \cref{fig:rhm_corr_schematic}). 
Both costs grow \emph{exponentially} with the depth $L$ of the latent hierarchy (and thus polynomially in the input dimension, i.e., sequence length).
Against this baseline, we establish three results:\looseness=-1

\begin{enumerate}
    \item \textbf{Efficient algorithm for representation learning} (\Cref{sec:latent-clustering})\textbf{.} We show that a hierarchical clustering algorithm recovers the full non-root latent tree from $\widetilde O (m^3)$ samples -- \textit{independent} of the depth \(L\), and hence exponentially fewer than those required by token-level SSL.\footnote{\label{footnote:code_availability} Code and data are available at \url{github.com/pcsl-epfl/learn-from-latents-not-tokens}. }\looseness=-1
    \item \textbf{An end-to-end stacked architecture with the same scaling} (\Cref{sec:neuralclustering})\textbf{.}
    We propose a new architecture where a stack of small predictor-clusterer modules predicts its own latents at each level. Trained end-to-end with gradient descent, it reaches the $m^3$ scaling on the RHM.\footnoteref{footnote:code_availability}
    \item \textbf{A theory of data2vec} (\Cref{sec:data2vec})\textbf{.} We show empirically and explain theoretically with reasonable assumptions that data2vec~\citep{baevskiData2vecGeneralFramework2022} implicitly
    performs hierarchical latent prediction, again leading to a sample complexity of order $m^3$. To our
    knowledge, this is the first sample-complexity analysis of such methods.
\end{enumerate}
{\it These results establish that learning from one's own latents can yield
huge sample-complexity gains over token-level objectives.} They also show that a single
latent-prediction module such as data2vec is in effect already performing multi-scale latent
discovery. Such methods may already be implicitly hierarchical,
weakening the case for explicit stacking such as H-JEPA.

\vspace{-1em}

\paragraph{Related work.}

The compositional and hierarchical structure of natural data has long been put forward as a key feature making it learnable by neural networks. Deep networks represent hierarchically compositional functions with exponentially fewer parameters than shallow ones \citep{poggio2017why}, yielding nonparametric sample-complexity bounds polynomial in the input dimension for both supervised \citep{schmidt2020nonparametric} and self-supervised \citep{mei2024unets} tasks. These results characterize expressivity and information-theoretic statistics, but do not describe what gradient descent can actually learn from finite data. A complementary line shows that transformers can approximately implement parsing-style inference on context-free grammars \citep{allenzhu2024cfg, zhao2023transformers, garnierbrun2024transformers}, characterizing the algorithm trained models implement, but not its sample complexity. The RHM \citep{cagnetta2024deep} makes simplifying assumptions -- frozen grammatical tree \citep{mossel2016deep,malach2018provably} and random production rules  \citep{degiuli2019random} -- that allow to fill this gap,  enabling sample-complexity predictions that match observations in deep nets. Existing RHM analyses cover supervised learning \citep{cagnetta2024deep} and token-level SSL \citep{cagnetta2024towards,favero2025compositional}, but not learning from one's own latents, which is the focus of this work.

\begin{figure}
    \vspace{-1em}
    \centering 
    \includegraphics[width=\linewidth]{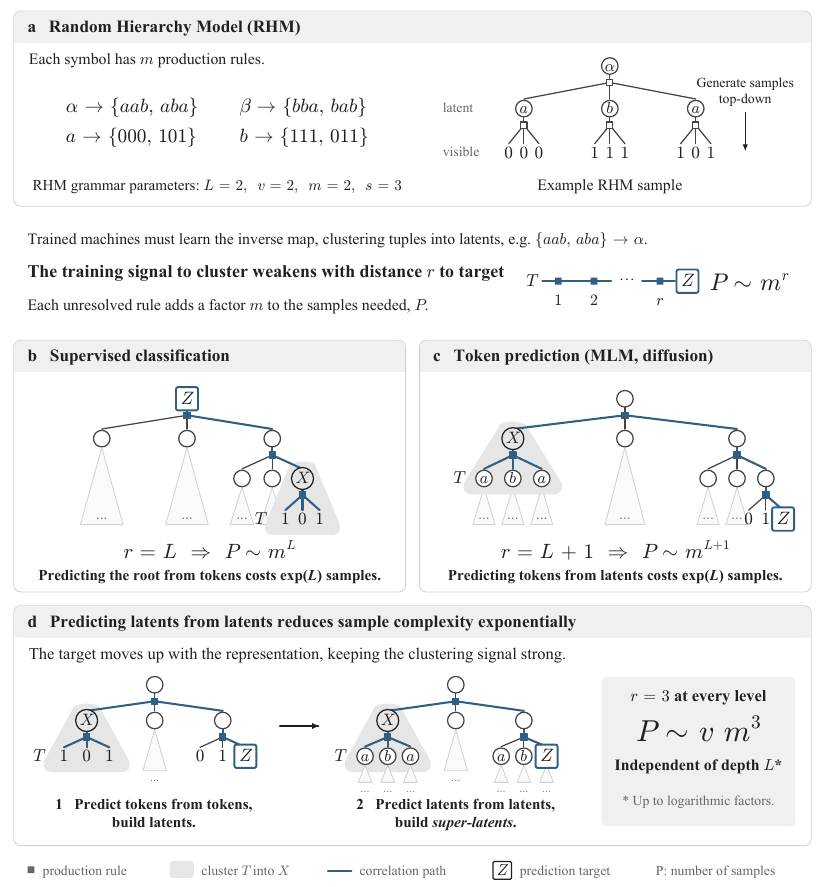}
    \caption{
        \textbf{Learning from latents reconstructs the RHM hierarchy more efficiently than token-level objectives.}
        \textbf{a.} An example Random Hierarchy Model (RHM) grammar (left) and a generated sample (right), with $L=2$. The RHM generates tokens by recursively choosing one of $m$ equiprobable production rules per symbol from a fixed grammar of $v$ symbols per level. Each rule expands a symbol into $s$ children; $L$ rounds give $s^L$ visible tokens. Trained on this data, token and root predictors learn to cluster synonymic tuples into latents in order to achieve good performance. 
        \textbf{b--c.} For supervised classification (b) and token-level SSL (c), we highlight the hardest latents to learn. These are the latents farthest from the training objective. The path length $r$ counts production rules (squares), each adding a factor $m$ to the samples required to learn those latents. Triangles abbreviate subtrees for visual clarity.
        \textbf{d.} Predicting tokens from tokens builds latents; predicting latents from latents builds super-latents. The distance remains $r=3$ at every level. Iterating recovers the non-root hierarchy with $\widetilde O(vm^3)$ samples under the assumptions of \Cref{thm:iterative-latent-clustering}, independent of depth (up to logarithmic factors).
        \label{fig:rhm_corr_schematic}
    }
\end{figure}

\section{Preliminaries}
\label{sec:rhm}

\paragraph{The Random Hierarchy Model (RHM).}

The RHM \citep{cagnetta2024deep} is a grammar whose rules are drawn at random. There are $L+1$ levels: the bottom one is what the learner sees, the $L$ above it are hidden. Each level has its own vocabulary of $v$ symbols. A \emph{rule} says that one symbol may be replaced by a particular ordered tuple of $s$ symbols from the level below. Of the $v^s$ tuples that could appear at a given level, only $vm$ are grammatical, chosen uniformly at random and distributed evenly, $m$ to each of the $v$ symbols above: every symbol has exactly $m$ ways to expand, and no two symbols share an expansion, so that every grammatical tuple identifies its parent unambiguously. In the grammar of \Cref{fig:rhm_corr_schematic}--\textbf{top} ($L=2$, $v=2$, $m=2$, $s=3$), the symbol $a$ can expand as $(0,0,0)$ or $(1,0,1)$, and $\alpha$ as $(a,a,b)$ or $(a,b,a)$. To generate a datum, draw one symbol uniformly at the top, replace it, via one of its $m$ rules chosen at random, by the $s$ children the rule specifies, do the same to each child with that level's rules. After $L$ rounds, the leaves form a string of $s^L$ visible tokens. The $m$ tuples sharing a parent, such as $(0,0,0)$
and $(1,0,1)$, are called \emph{synonyms}: exchanging one for another leaves the sentence grammatical and the derivation above the swap untouched. We denote the root of the tree by \(h^{(L)}_1\), its immediate descendants by \(h^{(L-1)}_1,\ldots,h^{(L-1)}_{s}\), and so on. Level $(0)$ denotes the visible leaf tokens, with \(h^{(0)}_i = x_i\).  When training neural networks, we randomly sample a grammar instance with RHM parameters (number of symbols \(v\), rules per symbol \(m\), tuple size \(s\), levels \(L\)) and then generate a dataset of \(P\) sentences from it. We refer the reader to \Cref{sec:rhm_formal} for the formal definition of the RHM grammar.

\paragraph{Correlation-based learning in the RHM.}

A central observation from \citet{cagnetta2024deep} is that learning the RHM amounts to learning invariances under the exchange of synonyms. 
The statistical signal that identifies these synonym classes comes from correlations between a tuple and an external observable. Concretely, let \(T^{(\ell)}\) be a  tuple of sibling symbols at level \(\ell\) and let \(Z\) be an observable elsewhere in the tree. This arrangement is depicted for different learning objectives in \Cref{fig:rhm_corr_schematic}. One can consider the connected correlation:
\[
    C_Z(\nu,z)
    :=
    \mathbb P[T^{(\ell)}=\nu,Z=z]
    -
    \mathbb P[T^{(\ell)}=\nu]\mathbb P[Z=z].
\]
If two tuples \(\nu,\nu^\prime \) are synonyms, then they have identical correlations with \(Z\):
\(
    C_Z(\nu,\cdot)=C_Z(\nu^\prime ,\cdot),
\)
whenever \(Z\) depends on the tuple only through its parent $X$.\footnote{Indeed, conditioned on the parent of the tuple, the rest of the tree is independent of which production rule was used to generate the tuple, and the uniform choice of production rules gives $\mathbb P[T^{(\ell)}=\nu]=\mathbb P[T^{(\ell)}=\nu^\prime]$ within each synonym class.} These identical correlations allow $\nu$ and $\nu^\prime$ to be identified as synonyms for the latent $X$, but can be obscured by finite-sampling noise. The signal-to-noise ratio of this correlation signal is attenuated by the tree distance between the tuple \(T^{(\ell)}\) and the observable $Z$. Along the path connecting them, each unresolved production rule averages over \(m\) synonymous choices, requiring a factor of \(m\) more samples to resolve the correlation. For both supervised classification and token-level self-supervision the overall sample complexity is therefore dominated by the longest path correlating a tuple with the target observable as illustrated in \Cref{fig:rhm_corr_schematic}.\looseness=-1

{\it Supervised classification \citep{cagnetta2024deep,ren2026provable}: } the observable \(Z\) is the class label at the root, and the hardest step is learning the rules mapping level-\(0\) tuples to level-\(1\) latents, requiring a sample complexity \( P^{\mathrm{sup}}\asymp m^L \). In contrast to shallow networks, which are cursed by dimensionality, deep networks trained on the supervised RHM classification task were found to saturate this scaling and reconstruct the latent hierarchy in their representations.

{\it Token-level self-supervision \citep{cagnetta2024towards,favero2025compositional, cagnetta2025scaling}: } the observable \(Z\) is a masked or noised token. In contrast to supervised learning, where the hardest step is the first one, token-level SSL requires only \( P_{\ell = 1}^{\mathrm{tok}} \asymp v m^3 \) samples to learn the first-level latents. Once these low-level latents have been internally reconstructed, the model can use them as coarse context variables. To learn higher levels, the relevant statistics are token--latent correlations: a visible token is correlated with a latent representation of its context through a longer tree path, leading to sample complexity \begin{equation}
    P_\ell^{\mathrm{tok}}
    \asymp
    v\,m^{\ell+2}.
    \label{eq:sample_complexity_token_latent}
\end{equation} 
for level \(\ell \geq 1\) latents.  Neural networks trained with token-level SSL objectives were empirically shown to saturate this staged scaling: as the dataset size increases, lower-level latents appear first, and higher-level latents become learnable thereafter. The overall sample complexity is gated behind reconstructing latents $\ell = L-1$ leading to $P_{\rm tok}\asymp vm^{L+1}$.\footnote{Notice that since the root is sampled uniformly and each root symbol has \(m\) equiprobable rules, the distribution does not reveal how the \(vm\) valid top-level tuples are partitioned into \(v\) groups of size \(m\), one group per root symbol. Thus unsupervised learning can recover latents $h^{(1)},h^{(2)},\ldots,h^{(L-1)}$ and the support of valid top-level tuples, but not the root labels themselves.}

\section{Recovering the RHM hierarchy by clustering}
\label{sec:latent-clustering}

\begin{figure}
    \centering
    \includegraphics[width=\linewidth]{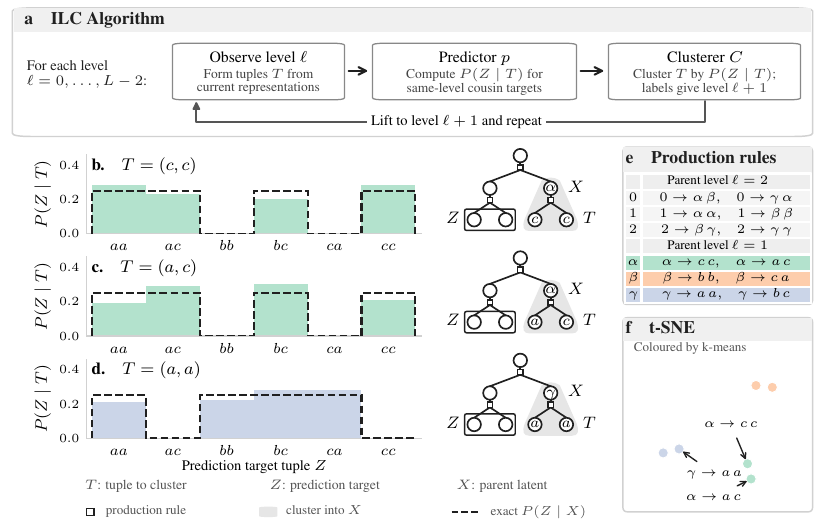}
    \vspace{-1em}
    \caption{\label{fig:ilc_alg} \textbf{The Iterative Latent Clustering (ILC) algorithm.}
    \textbf{a.} At each level, predict cousin targets from tuples, cluster tuples by their prediction vectors, and use the cluster labels as next-level latents.
    \textbf{b--d.} Empirical context vectors $P(Z\mid T)$ (filled histograms) and exact parent-conditioned distributions $P(Z\mid X)$ (dashed outlines) for synonymous tuples $(c,c)$ and $(a,c)$ (\textbf{b} and \textbf{c} respectively) with shared latent parent $\alpha$ and the same exact predictions; $(a,a)$ (\textbf{d}) has parent $\gamma$ and correspondingly a different context vector. 
    \textbf{e.} The fixed production rules generating the example.
    \textbf{f.} t-SNE of the full empirical context vectors, coloured by k-means assignment; arrows identify the three selected tuples. The same colours are used in the histograms and rule table. Example parameters are $P=500$, $L=2$, $s=2$, $v=3$, and $m=2$.}
    \vspace{-1em}
\end{figure}

As discussed in the previous section, the limitation of token-level objectives is that learning latents requires detecting ever weaker latent-token correlations as one moves up the hierarchy. 
The learn-from-your-own-latent setting removes this bottleneck: once level \(\ell\) has been recovered, both the conditioning object and the prediction target are lifted to level \(\ell\). The next stage is then again a local synonym-clustering problem, but now with the same statistical strength at every scale.
In other words, every level becomes as easy as the first token-level step. Given \Cref{eq:sample_complexity_token_latent}, this suggests that the whole non-root hierarchy should be recoverable from \(P\asymp vm^3\) samples.
This is illustrated in \Cref{fig:rhm_corr_schematic}d.\looseness=-1

\paragraph{Iterative latent-clustering algorithm (ILC).}
The preceding section framed learning the RHM as identifying synonyms through their identical correlations with an external observable. Here we cast this as a clustering problem. Let $T$ be a level-$\ell$ tuple with parent $X$ at level $\ell+1$, and let $Z$ be a level-$\ell$ symbol in a \emph{cousin} tuple of $T$, i.e., a tuple whose parent is a sibling of $X$, so that $Z$ and $T$ share their level-$(\ell+2)$ grandparent but not their parent (see \Cref{fig:ilc_alg}). Let $e_Z\in\mathbb{R}^v$ be the one-hot encoding of $Z$. For each grammatical tuple $\nu$, define the \emph{context vector}
\begin{equation*}
  \phi_\ell(\nu) \;:=\; \mathbb{E}[e_Z \mid T=\nu]
  \;=\; \big(\mathbb{P}[Z=z \mid T=\nu]\big)_{z},
\end{equation*}
i.e., the conditional distribution of the cousin symbol given the tuple.

Each grammatical tuple $\nu$ has a unique parent symbol $\mu=h_\ell(\nu)$, and, conditioned on $X$, the rest of the tree is independent of which of the $m$ rules generated $T$. Hence, at the population level, the $vm$ context vectors take only $v$ values, the \emph{parent context vectors} $\Phi_\ell(\mu) = \mathbb{P}[Z=\cdot \mid X=\mu]$ (black dashed lines in \Cref{fig:ilc_alg}). Synonym classes are identifiable provided these $v$ centers are distinct.

In practice one only has access to the \emph{empirical} context vectors $\widehat\phi_\ell(\nu)$, obtained by averaging $e_Z$ over the training samples. They fluctuate around $\Phi_\ell(h_\ell(\nu))$ and are resolved once their sampling noise falls below the separation between centers. Clustering the $vm$ empirical context vectors into $v$ groups then recovers the synonym classes, and the cluster index of a tuple defines its level-$(\ell+1)$ latent.

\Cref{alg:ilc} (Iterative Latent Clustering, ILC) operationalizes this procedure. Starting from the visible tokens, at each step $\ell$ the algorithm takes the decoded level-$\ell$ sequence, forms its tuples, estimates their empirical context vectors, clusters them, and relabels each tuple by its cluster to obtain level $\ell+1$. This clustering procedure is illustrated graphically in \Cref{fig:ilc_alg}: the predictor \(p\) estimates context vectors, while the clusterer \(C\) maps these vectors to next-level latent labels. \Cref{thm:iterative-latent-clustering} proves that $P\gtrsim vm^3$ samples suffice to reconstruct the hierarchy, up to logarithmic factors, and \Cref{fig:kmeans_and_slc} shows numerically that this threshold is tight.

\paragraph{Theoretical sample complexity.} An RHM grammar is \emph{balanced} if, for every level $\ell$, every grammatical tuple occurs with probability of order $1/(vm)$. It is \emph{separated} if every pair of parent context vector \( \Phi_{\ell}(\mu), \Phi_{\ell}(\mu') \) are separated by distance \( \gtrsim 1/m\). We refer the reader to Appendix~\ref{app:clustering-proof} for the formal statements of these assumptions and their justification in the limit of $v\to\infty$ at fixed $f \in(0,1)$, where \(f:=mv/v^{s}\) is the fraction of possible $s$-tuples allowed by the grammar.
For a balanced and separated RHM grammar with $f<1$, and assuming the existence of a \emph{stable} clustering module (see Appendix~\ref{app:clustering-proof}), we have: 

\begin{theorem}[Recovery of the non-root hierarchy; informal]
\label{thm:iterative-latent-clustering}
Fix an RHM grammar that is balanced and separated, and a clustering algorithm stable under perturbations. Then there exists a constant \(C>0\), depending only on the constants in those assumptions and on \(s\), such that the following holds. If
\begin{equation}
\label{eq:full-recovery-sample-complexity}
    P
    \geq
    C
    \left[
        vm\log\frac{Lvm}{\delta}
        +
        \frac{v m^3}{1-f}
        \log\frac{Lvm}{\delta}
    \right],
\end{equation}
then the iterative latent-clustering algorithm recovers
\(
    h^{(1)},h^{(2)},\ldots,h^{(L-1)}
\)
and all production-rule classes below the root with probability at least \(1-\delta\), up to independent permutations of the latent labels at each level.

In particular, for fixed \(f\) bounded away from \(1\) and up to logarithmic factors,
\(
    P_{\mathrm{ILC}}
    = \widetilde O(
    v m^3).
\)
\end{theorem}

In Appendix~\ref{app:clustering-proof}, we give the formal theorem and its proof, which is a level-by-level induction based on concentration of the empirical context vectors.

\paragraph{Numerical validation.}

We generate samples from the fixed-tree RHM and apply the level-by-level procedure above, estimating the empirical context vectors \(\widehat \phi_\ell(\nu)\) for observed grammatical tuples and clustering them using \(k\)-means (with $k=v$). %
Figure~\ref{fig:kmeans_and_slc}--\textbf{left} shows the reconstruction accuracy for \(L=5\), \(s=3\), \(v=8\), and varying \(m\). In the unrescaled plot, the transition to perfect reconstruction shifts to larger sample sizes as \(m\) increases. After rescaling the number of samples by
\(
     v m^3,
\)
the curves collapse and reconstruction becomes accurate once \(P/ v m^3\) %
is of order one, confirming the predicted \(m^3\) scaling of the local clustering threshold. \looseness=-1

\begin{figure}
    \includegraphics[width=\linewidth]{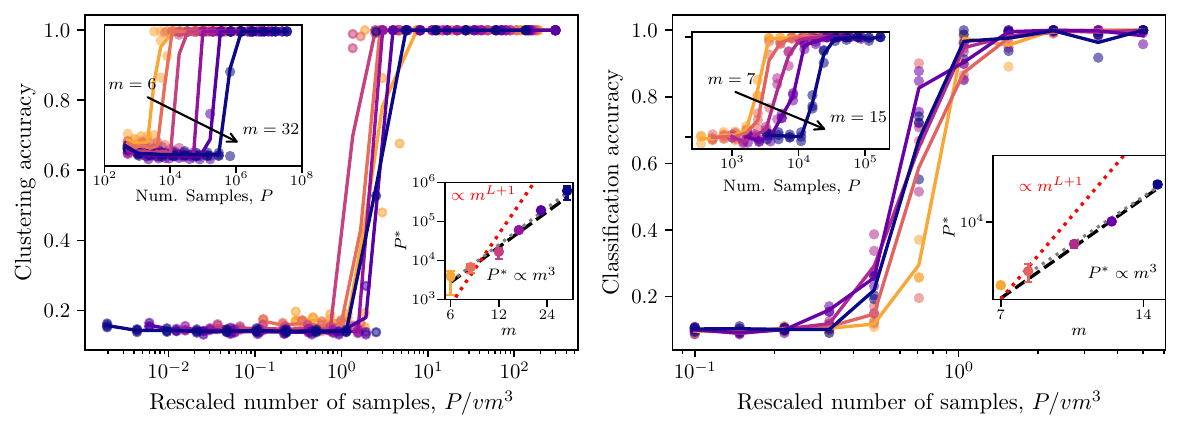}      
    \vspace{-2em}
    \caption{
    \textbf{Recursive clustering recovers the RHM hierarchy at the predicted scale.} 
    \textbf{Left}: Clustering reconstruction accuracy of the iterative clustering algorithm using \(k\)-means ($L=5,s=3,v=8$). \textbf{Right}: Class reconstruction accuracy of SLC architecture ($L=4,s=3,v=10$) trained end-to-end, evaluated by linear probe on the highest level latents $\widehat{h}^{(L-1)}$. In each panel, the main plot is a  {\it scaling collapse},  effected by rescaling number of samples by predicted sample complexity $P\sim vm^3$. Insets show raw data (top left) and inferred sample complexity (bottom right)  $P^\ast$ for which accuracy reaches 50\%. Monte-Carlo fitting $P^\ast \propto m^\alpha$ (grey-dots) gives $\alpha = 3.1\pm 0.3$ (ILC) and $\alpha = 2.9^{+0.4}_{-0.3}$ (SLC), see \Cref{app:statistical_fitting} for details. \looseness=-1    }
    \label{fig:kmeans_and_slc}
    \vspace{-1em}
\end{figure}

\section{Gradient-based Iterative Latent Clustering}
\label{sec:neuralclustering}

In this section, we ask whether the same scaling of the ILC algorithm can be achieved by a differentiable architecture trained end-to-end by gradient descent.

\paragraph{The Stacked Latent-Clustering (SLC) network.}
We mirror the recursion of ILC with a stack of $L-1$ identical modules. Each module is composed of two subnetworks:

\emph{Predictor} $p$:  $\mathrm{Pred}^{(\ell)}$ takes an $s$-tuple of level-$\ell$ tokens and is trained to output
a categorical distribution over its cousin tokens; this is the neural analog of the empirical cousin context vector. 

\emph{Clusterer} $C$:  $\mathrm{Clust}^{(\ell)}$ then maps each prediction vector to a soft assignment in a discrete codebook through a contrastive objective: prediction vectors with high similarity are pulled to the same code while dissimilar ones are pushed apart. The number of clusters is not fixed but learned from the data, making it more general than the k-means implementation of ILC.  %

The cluster assignments output by $C$ 
become the input tokens of the level-$(\ell+1)$ module. To keep the difficulty of the prediction task constant across depth, cluster outputs are tokenized with a softmax. Prediction targets are computed by a teacher network whose weights track an exponential moving average of the student,
\(
    W_\mathrm{teacher} \leftarrow (1-\alpha_\mathrm{ema})\,W_\mathrm{teacher}
    + \alpha_\mathrm{ema}\,W_\mathrm{student},
\)
a common trick to  prevent the representation collapse common to self-distilled SSL \citep{balestriero2023cookbook}. 
Full architectural and training details are given in Appendix~\ref{sec:app:slc_arch_details}.

\paragraph{Results.}

\Cref{fig:kmeans_and_slc}--\textbf{right} shows the accuracy of a linear probe trained to recover the top-level latent from the frozen final-layer representations of an SLC network, for $L=4$, $s=3$, $v=10$, and varying $m$. Rescaling the sample axis by $vm^3$ collapses the curves, showing that the network reaches the scaling of \Cref{thm:iterative-latent-clustering}. \Cref{fig:slc_L_scaling} shows that the sample complexity for $L\in \{3, \ldots, 7 \}$ does not vary.
Interestingly, inserting a stop-gradient at every module boundary (or even between $p$ and $C$), so that each level is trained \emph{only} by its own prediction and clustering losses, leaves the $vm^3$ scaling unchanged
(\Cref{fig:slc_stop_grads}). SLC therefore admits a strictly local learning rule, consistent with biologically motivated learning schemes.

\section{Analysis of data2vec}
\label{sec:data2vec}

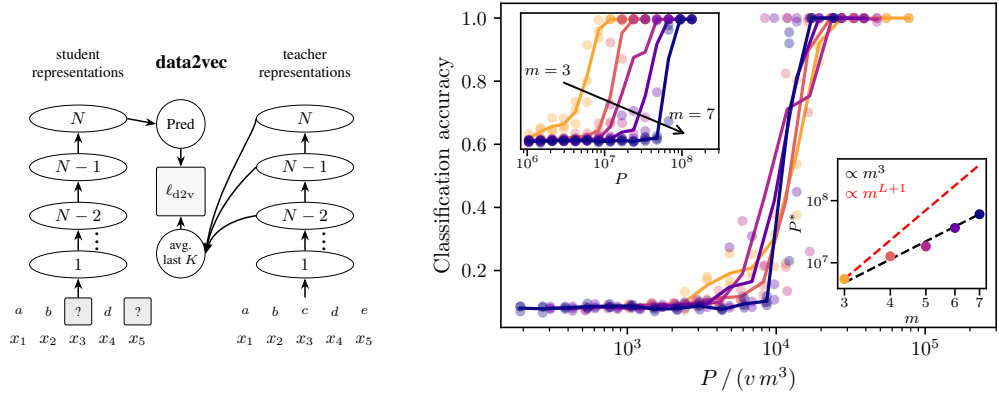
\begin{figure}
    \vspace{-1em}
    \centering
    \raisebox{-0.5\height}{\resizebox{!}{4cm}{\begin{tikzpicture}[
    >={Stealth[length=2.0mm]},
    every node/.style={font=\small},
    rep/.style={ellipse, draw, minimum width=18mm, minimum height=5mm, inner sep=1pt},
    token/.style={minimum width=5mm, minimum height=5mm, inner sep=1pt, font=\scriptsize},
    masked/.style={token, draw, rounded corners=1pt, fill=gray!15},
    pred/.style={circle, draw, minimum size=9mm, inner sep=0pt, fill=white, font=\footnotesize},
    loss/.style={draw, rectangle, rounded corners=1pt, minimum width=9mm, minimum height=9mm, inner sep=1pt, fill=gray!7, font=\footnotesize},
    avg/.style={circle, draw, minimum size=9mm, inner sep=0pt, font=\scriptsize, align=center, fill=white},
    arr/.style={->, thick},
    targ/.style={<-, thick}
]
    \path[use as bounding box] (-0.55,-1.00) rectangle (5.90,4.6);
    \node[font=\bfseries] at (2.75,4.30) {data2vec};
    \node[align=center, font=\footnotesize] at (0.65,4.25) {student\\representations};
    \node[align=center, font=\footnotesize] at (4.85,4.25) {teacher\\representations};

    \foreach \x/\sym/\idx in {-0.45/a/1,0.10/b/2,0.65/?/3,1.20/d/4,1.75/?/5} {
        \ifnum\idx=3
            \node[masked] (s\idx) at (\x,-0.32) {?};
        \else
            \ifnum\idx=5
                \node[masked] (s\idx) at (\x,-0.32) {?};
            \else
                \node[token] (s\idx) at (\x,-0.32) {$\sym$};
            \fi
        \fi
        \node[below=2pt of s\idx] {$x_\idx$};
    }
    \foreach \x/\sym/\idx in {3.75/a/1,4.30/b/2,4.85/c/3,5.40/d/4,5.95/e/5} {
        \node[token] (t\idx) at (\x,-0.32) {$\sym$};
        \node[below=2pt of t\idx] {$x_\idx$};
    }

    \foreach \y/\name/\lab in {3.25/l1/N,2.35/l2/{N-1},1.45/l3/{N-2},0.55/l4/1} {
        \node[rep] (\name s) at (0.65,\y) {$\lab$};
        \node[rep] (\name t) at (4.85,\y) {$\lab$};
    }
    \draw[arr] (l4s.north) -- (l3s.south);
    \draw[arr] (l3s.north) -- (l2s.south);
    \draw[arr] (l2s.north) -- (l1s.south);
    \draw[arr] (l4t.north) -- (l3t.south);
    \draw[arr] (l3t.north) -- (l2t.south);
    \draw[arr] (l2t.north) -- (l1t.south);
    \node[font=\small\bfseries] at (1.00,1.08) {$\vdots$};
    \node[font=\small\bfseries] at (5.20,1.08) {$\vdots$};

    \node[pred] (pred) at (2.55,3.15) {Pred};
    \node[loss] (mse) at (2.55,1.9) {$\ell_\mathrm{d2v}$};
    \node[avg] (avg) at (2.55,0.75) {avg.\\last $K$};

    \draw[arr] (l1s.east) -- (pred.west);
    \draw[arr] (pred.south) -- (mse.north);
    \draw[arr] (avg.north) -- (mse.south);
    \draw[arr] (l1t.west) .. controls (3.95,3.00) and (3.20,2.00) .. (avg.east);
    \draw[arr] (l2t.west) .. controls (3.90,2.30) and (3.20,1.70) .. (avg.east);
    \draw[arr] (l3t.west) .. controls (3.90,1.45) and (3.20,1.35) .. (avg.east);
    \draw[arr] (s3.north) -- (l4s.south);
    \draw[arr] (t3.north) -- (l4t.south);
\end{tikzpicture}}}%
    \hspace{2em}%
    \raisebox{-0.5\height}{\includegraphics[height=5.5cm]{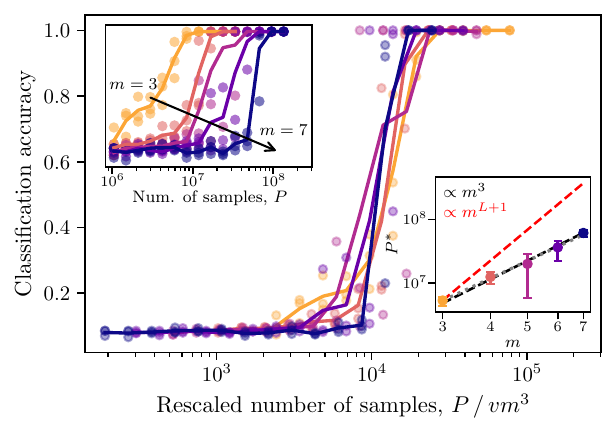}}
    \caption{\textbf{Data2vec on RHM.} 
    \textbf{Left:} data2vec learns in a teacher-student setting; the student is supervised on reconstruction of pooled teacher representations at masked token positions. 
    \textbf{Right:} Root-classification accuracy of probes on pooled final-layer features, as a
    function of the number of online pretraining samples $P$. RHM parameters: $v=16$, $s=2$, $L=4$. Insets and main are as in \Cref{fig:kmeans_and_slc}, statistical fit $P^\ast \sim m^\alpha$  (grey dots) yields $\alpha=2.8^{+0.3}_{-0.4}$.
    }
    \vspace{-1em}
    \label{fig:d2vec_online_class}
\end{figure}

In this section, we argue that the iterative clustering mechanism is \emph{implicit} in data2vec~\citep{baevskiData2vecGeneralFramework2022}, a popular SSL method that predicts its own latent representations. We choose data2vec over the closely related JEPA family because it comes with a published recipe for discrete-token inputs. In particular, data2vec uses a student-teacher setup. The student network is fed a partially masked input $S$; the teacher sees the unmasked input $x$. At every masked position $i$, the student is trained to predict the target $Y_i(x)$ given by the average of the last $K$ block activations of a teacher network applied to the unmasked input $x$, as illustrated in \Cref{fig:d2vec_online_class}--\textbf{left}. We follow the data2vec recipe for text, which trains the student with a smooth-L1 loss. Our theoretical analysis idealizes this loss as a squared loss, to which it is proportional whenever every per-coordinate residual is below a threshold of order 1.\footnote{Replacing the smooth-L1 loss by a pure squared loss leaves the $vm^3$ scaling unchanged (Appendix \ref{app:data2vec}, \Cref{fig:app:d2vec_controls}), so this idealization does not affect our conclusions.} As with SLC, the teacher has the same architecture as the student, and its weights track an exponential moving average of the student's weights, updated after each gradient step.

\paragraph{Empirical sample complexity.}

We pretrain data2vec on the RHM.  We follow the data2vec text recipe: a fraction $p_{\rm mask}=0.15$ of positions are masked. Full details are presented in Appendix~\ref{app:data2vec}. We compare two data-presentation regimes: \emph{online}, where each step draws a fresh batch from the grammar, and \emph{offline}, where a fixed training set of $P$ unique strings is reused across epochs. After pretraining, we freeze the encoder and train a one-hidden-layer MLP probe to predict the root label from mean-pooled final-layer features.

\Cref{fig:d2vec_online_class}--\textbf{right} shows the downstream root-classification accuracy of the probes for $L=4$, $s=2$, $v=16$ in the online setting.\footnote{The same plots are reported in Appendix~\ref{app:data2vec} for the offline setting.} The learning curves collapse onto a single curve after rescaling the sample axis by $vm^3$. Root classification on the RHM requires resolving all $L-1$ non-root latent levels and then identifying the partition of top-level tuples into root classes. As discussed in \Cref{sec:rhm}, token-level SSL is bottlenecked by learning the level-$(L-1)$ rules from tokens, which costs $\sim vm^{L+1}=vm^5$ samples; an additional $\mathcal O(vm)$ labeled examples then suffice to learn the root partition.
The observed scaling is hence incompatible with this token-level prediction baseline. We verify this directly by training parameter-matched transformers on the same data with a masked-language-modeling objective. Their sample complexity follows the token-level scaling $vm^{L+1}$ rather than $vm^3$ (see Appendix~\ref{app:data2vec}).
 
\paragraph{Theoretical analysis.}
 
We now explain the $vm^3$ scaling by tying data2vec back to the iterative latent clustering mechanism of \Cref{sec:latent-clustering}. 
The argument proceeds inductively by idealizing the teacher EMA as a discrete sequence of phases: within each phase, the teacher is held fixed and the student builds a new higher level of latent variable; between phases the teacher is refreshed to absorb new latents acquired by the student, enriching the target representations of the student. 
This is precisely the iterative latent clustering mechanism of \Cref{sec:latent-clustering}, and hence the $vm^3$ scaling follows. \looseness=-1

\paragraph{Assumptions.} We formalize the above argument through two assumptions, the first of which we verify numerically in \Cref{fig:app:d2vec_levels}.\\
\emph{(A1) Targets carry the learned latents.} Let $z_i^{(\ell)}:=h^{(\ell)}_{\lceil i/s^\ell\rceil}$ be the level-$\ell$ ancestor of position $i$, and let $e_z\in\mathbb R^v$ denote the one-hot encoding of a latent symbol $z$. If the encoder linearly represents the latents $h^{(1)},\ldots,h^{(\ell)}$, then the teacher target admits a decomposition
\[
    Y_i(x)
    \;=\;
    F_i(S)
    \;+\;
    \sum_{a=0}^{\ell}
        B_a\,e_{z_i^{(a)}}
    \;+\;
    {\rm residual},
\]
where $F_i(S)$ collects components determined by the visible inputs $S$ and the residual term collects everything else, which is typically nonlinear and not used in our analysis. We require that the linear maps $B_a$ keep different distributions over symbols apart: there is a constant $c_B>0$, independent of $v$, $m$, $L$, such that for any two probability vectors $p,q$,
\(
  \|B_a\,(p-q)\|_2 \;\ge\; c_B\,\|p-q\|_2
\). 
The residual architecture of the transformer makes the decomposition natural: a feature decodable from one block propagates through the identity path to every later block, and hence appears with a non-zero linear coefficient in the layers entering the top-$K$ teacher average. At initialization, $\ell=0$ and the only ancestor is $x_i$ itself, so $Y_i(x)$ contains a linear component of the masked one-hot inputs. We verify this assumption numerically in \Cref{fig:app:d2vec_levels}. \looseness=-1

\emph{(A2) Correlation learning.} Whenever a correlation between the prediction target and a feature of the visible input becomes detectable above sampling noise, gradient descent extracts that feature. This is the same correlation-learning hypothesis used in previous RHM analyses \citep{cagnetta2024towards,favero2025compositional}, where it is supported empirically and by one-step gradient calculations.

\paragraph{Phase 0.} The population minimizer of the data2vec loss is the conditional expectation $\mathbb E[Y_i(x)\mid S]$. By linearity of conditional expectation, in phase $0$, predicting $Y_i(x)$ from $S$ reduces, up to the visible component and the non-linear residual, to predicting the masked one-hot $e_{x_i}$ from its visible context. This is exactly the token-level masked-prediction problem analyzed in \Cref{sec:rhm}: by (A1), $B_0$ preserves distances between conditional token distributions up to the constant $c_B$, so local token--token correlations become detectable for $P\gtrsim vm^3$, and by (A2) the network extracts them. The resulting representation collapses synonymous level-$0$ tuples onto a common image, and the level-$1$ ancestor becomes linearly decodable from the encoder.

\paragraph{Phase $\ell\geq 1$.} Suppose phases $0,\ldots,\ell-1$ have produced linearly decodable representations of $h^{(1)},\ldots,h^{(\ell)}$ inside the encoder. After the next teacher update, these latents enter the target, and (A1) gives the new component $\mathbb E[e_{z_i^{(\ell)}}\mid S]$ in the optimal predictor. Because levels $1,\ldots,\ell$ are decodable, the visible context $S$ can be parsed into level-$\ell$ symbols away from the masked position. In particular, it contains the level-$\ell$ cousin tuple $T_\ell$, the closest level-$\ell$ tuple sitting outside the production rule of position $i$. Conditioning on $T_\ell$ produces the context vector $\phi_\ell(\nu)=\mathbb{E}[e_{z_i^{(\ell)}}\mid T_\ell=\nu]$ studied in \Cref{sec:latent-clustering}. The prediction problem at scale $\ell$ has \emph{lifted} from leaves to level-$\ell$ tuples, and is now identical to the clustering problem solved by the ILC algorithm of that section, requiring $P\gtrsim vm^3$. When the network extracts this signal, the level-$(\ell+1)$ ancestor $z_i^{(\ell+1)}$ becomes linearly decodable, and the next teacher update carries it into the target.

The induction continues up to $\ell=L-2$. Every phase shares the same $vm^3$ sample complexity, so 
the offline sample complexity equals the per-phase threshold,
\(P_{\rm d2v} \;\asymp\; vm^3\),
and depth enters only through the number of phases. %
Finally, our argument also explains why the teacher--student loop does not collapse. Indeed, the target always contains something the student has not yet extracted (e.g., the masked token at
initialization or a new latent after each refresh), so the optimal prediction
$\mathbb{E}[Y_i(x)\mid S]$ varies with $S$ and a constant representation is not a
minimizer. Novelty then runs out only once the last level has been extracted, possibly leading to collapse at late times.\footnote{Notice that in continuous time the same holds provided the teacher lags the student by at least the time needed to learn one new level, so that the target still carries novelty.}

\begin{figure}
    \centering
    \vspace{-1em}
    \includegraphics[width=0.85\linewidth]{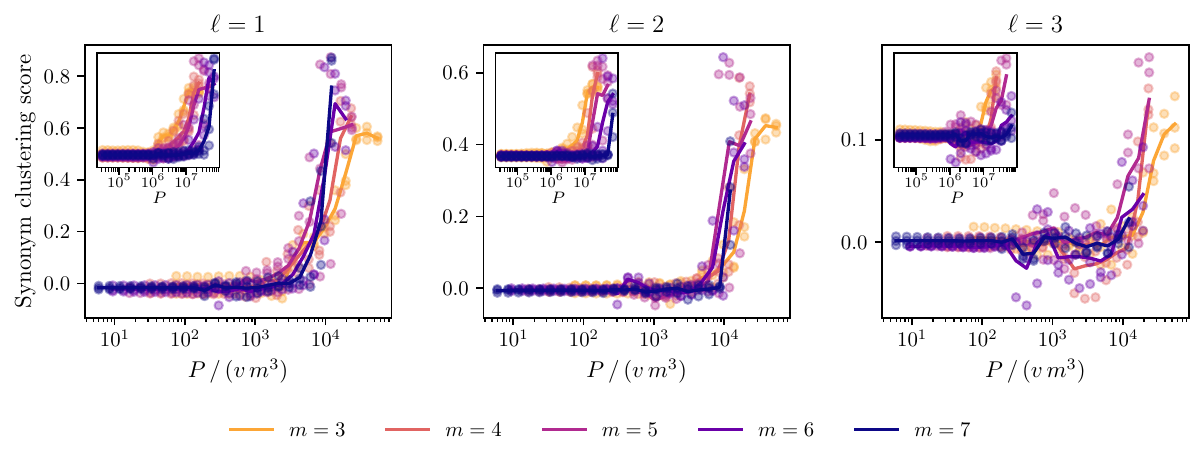}
    \vspace{-1em}
    \caption{\textbf{Data2vec clusters synonyms.} Synonym clustering score $\mathcal C_\ell$ of the encoder at $\ell\in\{1,2,3\}$, for $L=4$, $s=2$, $v=16$. A positive score means synonyms are mapped closer together than non-synonyms. \emph{Insets:} raw curves vs. $P$. \emph{Main:} same curves rescaled by $vm^3$ collapse.}
    \vspace{-.5em}
    \label{fig:d2vec-clustering}
\end{figure}

\paragraph{The encoder internally clusters synonyms.}

Our theoretical argument predicts that the data2vec encoder collapses synonymous level-$\ell$ tuples onto a common representation -- the very geometric signature of clustering. We test this prediction directly by measuring the \emph{synonym clustering score} $\mathcal C_\ell$ of the encoder at level $\ell$ by comparing the change in its hidden representations under two interventions on the input: a \emph{synonym swap}, which resamples every level-$\ell$ production rule and thus replaces each level-$(\ell-1)$ tuple by a synonym while leaving higher-level latents unchanged, and a \emph{generic swap} to an unrelated input. 
\Cref{fig:d2vec-clustering} reports $\mathcal C_\ell$ at each non-root level after the third block of the encoder as a function of pretraining samples $P$. The scores at all three levels rise from zero and collapse onto a common curve once the sample axis is rescaled by $vm^3$. This is direct evidence that data2vec's encoder implements the same recursive clustering mechanism as the ILC algorithm of \Cref{sec:latent-clustering}, with the same per-level $vm^3$ sample complexity predicted by the theory.\looseness=-1

\section{Conclusion}

We have shown that learning from one's own latents recovers the full non-root latent tree of the RHM from a number of samples scaling as $m^3$, exponentially fewer than the $m^{L+1}$ required by token-level self-supervised objectives. We proved this for a hierarchical clustering algorithm, showed empirically that an end-to-end neural network of predictor-clusterer modules reached the same scaling, and argued that data2vec does so by implicitly implementing the same mechanism. The main lesson is that latents at the same level of the hierarchy are far more correlated with each other than they are with raw tokens, so predicting from one's own latents amplifies a signal that token-level prediction dilutes. This gives a quantitative footing to the common intuition that token-level prediction is suboptimal.

\paragraph{Practical implications.}
Recent work supports that empirical neural scaling laws in language are governed by power-law decays of token correlations with context length~\citep{cagnetta2024towards, cagnettaDerivingNeuralScaling2026, favero2025compositional}. Can generative models trained with own-latent self-supervision systematically beat existing scaling laws, alone or in combination with token-level losses? A useful first step would be a controlled comparison between data2vec and a same-architecture next-token-prediction baseline as the training-set size $P$ is varied: at small $P$ the two should diverge sharply, while at very large $P$ they should converge to the same latent representations. Such a comparison would empirically test our central prediction and, if confirmed, motivate latent-supervised generative models as a route to break current scaling regimes.

\paragraph{Biological plausibility.}
As we show in Appendix~\ref{app:sec:slc_local_rules}, the predictor-clusterer modules of \Cref{sec:neuralclustering} can also reach the $m^3$ scaling of sample complexity using only quasilocal learning rules -- stop-gradients between modules in place of end-to-end backpropagation. In this local learning setting our method is conceptually similar to contrastive predictive coding~\citep{oordRepresentationLearningContrastive2019a} and CLAPP~\citep{illingLocalPlasticityRules2021}, which both propose biologically plausible self-supervised objectives that predict in latent rather than raw-sensory space. 
Does the remarkable sample efficiency of the human brain stem from this mechanism? 
It would be interesting to test whether such bio-plausible architectures can reach an $m^3$ sample complexity on the RHM simply by learning from their own latents.\looseness=-1

\paragraph{Limitations.}
The RHM is a deliberately simple synthetic language: its context-free grammar has a fixed tree topology and unambiguous, non-recursive production rules. Extending the analysis to (i) variable tree topologies, (ii) recursive production rules that can call themselves, and (iii) explicitly context-dependent rules will bring the theory closer to natural language and images. 
Although our analysis suggests that stacking JEPAs as in H-JEPA is largely redundant, very recent self-supervised architectures claim explicit hierarchical supervision but are neither stacked JEPAs nor stacked data2vecs, and instead resemble data2vec-style training procedures with multi-scale teachers~\citep{loweSelfdistillationHiddenLayers2026a, maesLeWorldModelStableEndtoend2026}. These architectures would be interesting to study theoretically. %

\newpage 
\begin{ack}
We thank Francesco Cagnetta, Francesco D'Amico, Ariane Delrocq, Antonio Sclocchi, and Wu Zihan for discussions and feedback on the manuscript, and the members of the Simons collaboration for discussions. D. J. Korchinski acknowledges financial support from the Natural Sciences and Engineering Research Council of Canada (NSERC PDF - 587940 - 2024). A. Favero acknowledges support from the Infosys-Cambridge AI Centre. This work was supported by the Simons Foundation through the Simons Collaboration on the Physics of Learning and Neural Computation (Award ID: SFI-MPS-POL00012574-05), PI  Wyart.
\end{ack}

\bibliographystyle{unsrtnat}
\bibliography{references}

\begin{thebibliography}{65}
\providecommand{\natexlab}[1]{#1}
\providecommand{\url}[1]{\texttt{#1}}
\expandafter\ifx\csname urlstyle\endcsname\relax
  \providecommand{\doi}[1]{doi: #1}\else
  \providecommand{\doi}{doi: \begingroup \urlstyle{rm}\Url}\fi

\bibitem[Ho et~al.(2020)Ho, Jain, and Abbeel]{ho2020denoising}
Jonathan Ho, Ajay Jain, and Pieter Abbeel.
\newblock Denoising diffusion probabilistic models.
\newblock In \emph{Advances in Neural Information Processing Systems}, 2020.

\bibitem[Rombach et~al.(2022)Rombach, Blattmann, Lorenz, Esser, and Ommer]{rombach2022high}
Robin Rombach, Andreas Blattmann, Dominik Lorenz, Patrick Esser, and Bj{\"o}rn Ommer.
\newblock High-resolution image synthesis with latent diffusion models.
\newblock In \emph{IEEE/CVF Conference on Computer Vision and Pattern Recognition (CVPR)}, 2022.

\bibitem[Brooks et~al.(2024)Brooks, Peebles, Holmes, DePue, Guo, Jing, Schnurr, Taylor, Luhman, Luhman, Ng, Wang, and Ramesh]{brooks2024video}
Tim Brooks, Bill Peebles, Connor Holmes, Will DePue, Yufei Guo, Li~Jing, David Schnurr, Joe Taylor, Troy Luhman, Eric Luhman, Clarence Wing~Yin Ng, Ricky Wang, and Aditya Ramesh.
\newblock Video generation models as world simulators.
\newblock OpenAI Technical Report, 2024.
\newblock URL \url{https://openai.com/research/video-generation-models-as-world-simulators}.

\bibitem[Brown et~al.(2020)Brown, Mann, Ryder, Subbiah, Kaplan, Dhariwal, Neelakantan, et~al.]{brown2020language}
Tom~B. Brown, Benjamin Mann, Nick Ryder, Melanie Subbiah, Jared Kaplan, Prafulla Dhariwal, Arvind Neelakantan, et~al.
\newblock Language models are few-shot learners.
\newblock In \emph{Advances in Neural Information Processing Systems}, 2020.

\bibitem[{OpenAI}(2023)]{openai2023gpt4}
{OpenAI}.
\newblock {GPT-4} technical report.
\newblock \emph{arXiv preprint arXiv:2303.08774}, 2023.

\bibitem[{DeepSeek-AI}(2025)]{deepseekai2025r1}
{DeepSeek-AI}.
\newblock {DeepSeek-R1}: Incentivizing reasoning capability in {LLMs} via reinforcement learning.
\newblock \emph{arXiv preprint arXiv:2501.12948}, 2025.

\bibitem[Hoffmann et~al.(2022)Hoffmann, Borgeaud, Mensch, Buchatskaya, Cai, Rutherford, de~Las~Casas, Hendricks, et~al.]{hoffmann2022training}
Jordan Hoffmann, Sebastian Borgeaud, Arthur Mensch, Elena Buchatskaya, Trevor Cai, Eliza Rutherford, Diego de~Las~Casas, Lisa~Anne Hendricks, et~al.
\newblock Training compute-optimal large language models.
\newblock In \emph{Advances in Neural Information Processing Systems}, 2022.

\bibitem[Touvron et~al.(2023)Touvron, Lavril, Izacard, Martinet, Lachaux, Lacroix, Rozi{\`e}re, Goyal, Hambro, Azhar, Rodriguez, Joulin, Grave, and Lample]{touvron2023llama}
Hugo Touvron, Thibaut Lavril, Gautier Izacard, Xavier Martinet, Marie-Anne Lachaux, Timoth{\'e}e Lacroix, Baptiste Rozi{\`e}re, Naman Goyal, Eric Hambro, Faisal Azhar, Aurelien Rodriguez, Armand Joulin, Edouard Grave, and Guillaume Lample.
\newblock {LLaMA}: Open and efficient foundation language models.
\newblock \emph{arXiv preprint arXiv:2302.13971}, 2023.

\bibitem[Gilkerson et~al.(2017)Gilkerson, Richards, Warren, Montgomery, Greenwood, Oller, Hansen, and Paul]{gilkerson2017mapping}
Jill Gilkerson, Jeffrey~A. Richards, Steven~F. Warren, Judith~K. Montgomery, Charles~R. Greenwood, D.~Kimbrough Oller, John H.~L. Hansen, and Terrance~D. Paul.
\newblock Mapping the early language environment using all-day recordings and automated analysis.
\newblock \emph{American Journal of Speech-Language Pathology}, 26\penalty0 (2):\penalty0 248--265, 2017.

\bibitem[Linzen(2020)]{linzen2020accelerate}
Tal Linzen.
\newblock How can we accelerate progress towards human-like linguistic generalization?
\newblock In \emph{Proceedings of the 58th Annual Meeting of the Association for Computational Linguistics}, 2020.

\bibitem[Frank(2023)]{frank2023bridging}
Michael~C. Frank.
\newblock Bridging the data gap between children and large language models.
\newblock \emph{Trends in Cognitive Sciences}, 27\penalty0 (11):\penalty0 990--992, 2023.

\bibitem[Schuhmann et~al.(2022)Schuhmann, Beaumont, Vencu, Gordon, Wightman, Cherti, Coombes, Katta, Mullis, Wortsman, Schramowski, Kundurthy, Crowson, Schmidt, Kaczmarczyk, and Jitsev]{schuhmann2022laion}
Christoph Schuhmann, Romain Beaumont, Richard Vencu, Cade Gordon, Ross Wightman, Mehdi Cherti, Theo Coombes, Aarush Katta, Clayton Mullis, Mitchell Wortsman, Patrick Schramowski, Srivatsa Kundurthy, Katherine Crowson, Ludwig Schmidt, Robert Kaczmarczyk, and Jenia Jitsev.
\newblock {LAION-5B}: An open large-scale dataset for training next generation image-text models.
\newblock In \emph{Advances in Neural Information Processing Systems Datasets and Benchmarks Track}, 2022.

\bibitem[Smith and Gasser(2005)]{smith2005development}
Linda~B. Smith and Michael Gasser.
\newblock The development of embodied cognition: Six lessons from babies.
\newblock \emph{Artificial Life}, 11\penalty0 (1--2):\penalty0 13--29, 2005.

\bibitem[Alayrac et~al.(2022)Alayrac, Donahue, Luc, Miech, Barr, Hasson, Lenc, Mensch, Millican, Reynolds, et~al.]{alayrac2022flamingo}
Jean-Baptiste Alayrac, Jeff Donahue, Pauline Luc, Antoine Miech, Iain Barr, Yana Hasson, Karel Lenc, Arthur Mensch, Katie Millican, Malcolm Reynolds, et~al.
\newblock {Flamingo}: a visual language model for few-shot learning.
\newblock In \emph{Advances in Neural Information Processing Systems}, 2022.

\bibitem[Rao and Ballard(1999)]{rao1999predictive}
Rajesh P.~N. Rao and Dana~H. Ballard.
\newblock Predictive coding in the visual cortex: a functional interpretation of some extra-classical receptive-field effects.
\newblock \emph{Nature Neuroscience}, 2\penalty0 (1):\penalty0 79--87, 1999.

\bibitem[Friston(2010)]{friston2010free}
Karl Friston.
\newblock The free-energy principle: a unified brain theory?
\newblock \emph{Nature Reviews Neuroscience}, 11\penalty0 (2):\penalty0 127--138, 2010.

\bibitem[LeCun et~al.(2022)]{lecun2022path}
Yann LeCun et~al.
\newblock A path towards autonomous machine intelligence version 0.9. 2, 2022-06-27.
\newblock \emph{Open Review}, 62\penalty0 (1):\penalty0 1--62, 2022.

\bibitem[Tack et~al.(2025)Tack, Lanchantin, Yu, Cohen, Kulikov, Lan, Hao, Tian, Weston, and Li]{tackLLMPretrainingContinuous2025}
Jihoon Tack, Jack Lanchantin, Jane Yu, Andrew Cohen, Ilia Kulikov, Janice Lan, Shibo Hao, Yuandong Tian, Jason Weston, and Xian Li.
\newblock {{LLM}} pretraining with continuous concepts.
\newblock \emph{arXiv preprint arXiv:2502.08524}, February 2025.
\newblock \doi{10.48550/arXiv.2502.08524}.

\bibitem[Liu et~al.(2026)Liu, Song, Wang, Ge, Lamb, Guo, Chen, Zhou, and Lin]{liuNextConceptPrediction2026}
Yuliang Liu, Yunchong Song, Yixuan Wang, Kewen Ge, Alex Lamb, Qipeng Guo, Kai Chen, Bowen Zhou, and Zhouhan Lin.
\newblock Next concept prediction in discrete latent space leads to stronger language models.
\newblock \emph{arXiv preprint arXiv:2602.08984}, February 2026.
\newblock \doi{10.48550/arXiv.2602.08984}.

\bibitem[Illing et~al.(2021)Illing, Ventura, Bellec, and Gerstner]{illingLocalPlasticityRules2021}
Bernd Illing, Jean Ventura, Guillaume Bellec, and Wulfram Gerstner.
\newblock Local plasticity rules can learn deep representations using self-supervised contrastive predictions.
\newblock In \emph{Advances in {{Neural Information Processing Systems}}}, volume~34, pages 30365--30379. Curran Associates, Inc., 2021.

\bibitem[Millidge et~al.(2022)Millidge, Seth, and Buckley]{millidge2022predictive}
Beren Millidge, Anil Seth, and Christopher~L. Buckley.
\newblock Predictive coding: a theoretical and experimental review.
\newblock \emph{arXiv preprint arXiv:2107.12979}, 2022.

\bibitem[Grill et~al.(2020)Grill, Strub, Altch{\'e}, Tallec, Richemond, Buchatskaya, Doersch, Pires, Guo, Azar, Piot, Kavukcuoglu, Munos, and Valko]{grillBootstrapYourOwn2020}
Jean-Bastien Grill, Florian Strub, Florent Altch{\'e}, Corentin Tallec, Pierre~H. Richemond, Elena Buchatskaya, Carl Doersch, Bernardo~Avila Pires, Zhaohan~Daniel Guo, Mohammad~Gheshlaghi Azar, Bilal Piot, Koray Kavukcuoglu, R{\'e}mi Munos, and Michal Valko.
\newblock Bootstrap your own latent a new approach to self-supervised learning.
\newblock In \emph{Proceedings of the 34th {{International Conference}} on {{Neural Information Processing Systems}}}, {{NIPS}} '20, pages 21271--21284, Red Hook, NY, USA, December 2020. Curran Associates Inc.
\newblock ISBN 978-1-7138-2954-6.

\bibitem[Caron et~al.(2021)Caron, Touvron, Misra, J{\'e}gou, Mairal, Bojanowski, and Joulin]{caronEmergingPropertiesSelfsupervised2021}
Mathilde Caron, Hugo Touvron, Ishan Misra, Herv{\'e} J{\'e}gou, Julien Mairal, Piotr Bojanowski, and Armand Joulin.
\newblock Emerging properties in self-supervised vision transformers.
\newblock In \emph{Proceedings of the {{IEEE}}/{{CVF International Conference}} on {{Computer Vision}}}, pages 9650--9660, 2021.

\bibitem[Baevski et~al.(2022)Baevski, Hsu, Xu, Babu, Gu, and Auli]{baevskiData2vecGeneralFramework2022}
Alexei Baevski, Wei-Ning Hsu, Qiantong Xu, Arun Babu, Jiatao Gu, and Michael Auli.
\newblock Data2vec: A general framework for self-supervised learning in speech, vision and language.
\newblock In \emph{Proceedings of the 39th {{International Conference}} on {{Machine Learning}}}, pages 1298--1312. PMLR, June 2022.

\bibitem[Assel et~al.(2025)Assel, Ibrahim, Biancalani, Regev, and Balestriero]{vanassel2025joint}
Hugues~Van Assel, Mark Ibrahim, Tommaso Biancalani, Aviv Regev, and Randall Balestriero.
\newblock Joint embedding vs reconstruction: Provable benefits of latent space prediction for self-supervised learning.
\newblock \emph{arXiv preprint arXiv:2505.12477}, 2025.

\bibitem[Li et~al.(2025)Li, Xue, Ao, Song, and Salim]{liHiTJEPAHierarchicalSelfsupervised2025}
Lihuan Li, Hao Xue, Shuang Ao, Yang Song, and Flora Salim.
\newblock {{HiT-JEPA}}: {{A Hierarchical Self-supervised Trajectory Embedding Framework}} for {{Similarity Computation}}.
\newblock \emph{arXiv preprint arXiv:2507.00028}, June 2025.
\newblock \doi{10.48550/arXiv.2507.00028}.

\bibitem[Girgis et~al.(2026)Girgis, Labriji, and Bennis]{girgisHierarchicalJEPAMeets2026}
Abanoub~M. Girgis, Ibtissam Labriji, and Mehdi Bennis.
\newblock Hierarchical {{JEPA}} meets predictive remote control in beyond {{5G}} networks.
\newblock \emph{arXiv preprint arXiv:2602.07000}, January 2026.
\newblock \doi{10.48550/arXiv.2602.07000}.

\bibitem[Booth and Thompson(1973)]{booth1973applying}
Taylor~L. Booth and Richard~A. Thompson.
\newblock Applying probability measures to abstract languages.
\newblock \emph{IEEE Transactions on Computers}, C-22\penalty0 (5):\penalty0 442--450, 1973.

\bibitem[Manning and Sch{\"u}tze(1999)]{manning1999foundations}
Christopher~D. Manning and Hinrich Sch{\"u}tze.
\newblock \emph{Foundations of Statistical Natural Language Processing}.
\newblock MIT Press, Cambridge, MA, 1999.

\bibitem[Grenander(1996)]{grenander1996elements}
Ulf Grenander.
\newblock \emph{Elements of Pattern Theory}.
\newblock Johns Hopkins University Press, 1996.

\bibitem[Mumford(1994)]{mumford1994pattern}
David Mumford.
\newblock Pattern theory: a unifying perspective.
\newblock In \emph{First European Congress of Mathematics, Vol.~I (Paris, 1992)}, volume 119 of \emph{Progress in Mathematics}, pages 187--224. Birkh{\"a}user, 1994.

\bibitem[Cagnetta et~al.(2024)Cagnetta, Petrini, Tomasini, Favero, and Wyart]{cagnetta2024deep}
Francesco Cagnetta, Leonardo Petrini, Umberto~M Tomasini, Alessandro Favero, and Matthieu Wyart.
\newblock How deep neural networks learn compositional data: The random hierarchy model.
\newblock \emph{Physical Review X}, 14\penalty0 (3):\penalty0 031001, 2024.

\bibitem[Cagnetta and Wyart(2024)]{cagnetta2024towards}
Francesco Cagnetta and Matthieu Wyart.
\newblock Towards a theory of how the structure of language is acquired by deep neural networks.
\newblock \emph{Advances in Neural Information Processing Systems}, 37:\penalty0 83119--83163, 2024.

\bibitem[Cagnetta et~al.(2026)Cagnetta, Ravent{\'o}s, Ganguli, and Wyart]{cagnettaDerivingNeuralScaling2026}
Francesco Cagnetta, Allan Ravent{\'o}s, Surya Ganguli, and Matthieu Wyart.
\newblock Deriving neural scaling laws from the statistics of natural language.
\newblock \emph{arXiv preprint arXiv:2602.07488}, 2026.

\bibitem[Sclocchi et~al.(2025{\natexlab{a}})Sclocchi, Favero, and Wyart]{sclocchiPhaseTransitionDiffusion2025}
Antonio Sclocchi, Alessandro Favero, and Matthieu Wyart.
\newblock A phase transition in diffusion models reveals the hierarchical nature of data.
\newblock \emph{Proceedings of the National Academy of Sciences}, 122\penalty0 (1):\penalty0 e2408799121, January 2025{\natexlab{a}}.
\newblock \doi{10.1073/pnas.2408799121}.

\bibitem[Sclocchi et~al.(2025{\natexlab{b}})Sclocchi, Favero, Itzhak~Levi, and Wyart]{sclocchiProbingLatentHierarchical2025}
Antonio Sclocchi, Alessandro Favero, Noam Itzhak~Levi, and Matthieu Wyart.
\newblock Probing the latent hierarchical structure of data via diffusion models.
\newblock \emph{Journal of Statistical Mechanics: Theory and Experiment}, 2025\penalty0 (8):\penalty0 84005, August 2025{\natexlab{b}}.
\newblock ISSN 1742-5468.
\newblock \doi{10.1088/1742-5468/aded6c}.

\bibitem[Favero et~al.(2025{\natexlab{a}})Favero, Sclocchi, Cagnetta, Frossard, and Wyart]{favero2025compositional}
Alessandro Favero, Antonio Sclocchi, Francesco Cagnetta, Pascal Frossard, and Matthieu Wyart.
\newblock How compositional generalization and creativity improve as diffusion models are trained.
\newblock In \emph{International Conference on Machine Learning}, pages 16286--16306. PMLR, 2025{\natexlab{a}}.

\bibitem[Favero et~al.(2025{\natexlab{b}})Favero, Sclocchi, and Wyart]{favero2025bigger}
Alessandro Favero, Antonio Sclocchi, and Matthieu Wyart.
\newblock Bigger isn't always memorizing: Early stopping overparameterized diffusion models.
\newblock \emph{arXiv preprint arXiv:2505.16959}, 2025{\natexlab{b}}.

\bibitem[Ren et~al.(2026)Ren, Dandi, Krzakala, and Lee]{ren2026provable}
Yunwei Ren, Yatin Dandi, Florent Krzakala, and Jason~D. Lee.
\newblock Provable learning of random hierarchy models and hierarchical shallow-to-deep chaining.
\newblock \emph{arXiv preprint arXiv:2601.19756}, 2026.

\bibitem[Cagnetta et~al.(2025)Cagnetta, Favero, Sclocchi, and Wyart]{cagnetta2025scaling}
Francesco Cagnetta, Alessandro Favero, Antonio Sclocchi, and Matthieu Wyart.
\newblock Scaling laws and representation learning in simple hierarchical languages: Transformers versus convolutional architectures.
\newblock \emph{Physical Review E}, 112\penalty0 (6):\penalty0 065312, 2025.

\bibitem[Poggio et~al.(2017)Poggio, Mhaskar, Rosasco, Miranda, and Liao]{poggio2017why}
Tomaso Poggio, Hrushikesh Mhaskar, Lorenzo Rosasco, Brando Miranda, and Qianli Liao.
\newblock Why and when can deep-but-not-shallow networks avoid the curse of dimensionality: A review.
\newblock \emph{International Journal of Automation and Computing}, 14\penalty0 (5):\penalty0 503--519, 2017.

\bibitem[Schmidt-Hieber(2020)]{schmidt2020nonparametric}
Johannes Schmidt-Hieber.
\newblock Nonparametric regression using deep neural networks with {ReLU} activation function.
\newblock \emph{Annals of Statistics}, 48\penalty0 (4):\penalty0 1875--1897, 2020.

\bibitem[Mei(2025)]{mei2024unets}
Song Mei.
\newblock {U-Nets} as belief propagation: Efficient classification, denoising, and diffusion in generative hierarchical models.
\newblock In \emph{International Conference on Learning Representations (ICLR)}, 2025.
\newblock arXiv:2404.18444.

\bibitem[Allen-Zhu and Li(2025)]{allenzhu2024cfg}
Zeyuan Allen-Zhu and Yuanzhi Li.
\newblock Physics of language models: Part 1, learning hierarchical language structures.
\newblock \emph{Transactions on Machine Learning Research}, 2025.
\newblock URL \url{https://openreview.net/forum?id=mPQKyzkA1K}.

\bibitem[Zhao et~al.(2023)Zhao, Panigrahi, Ge, and Arora]{zhao2023transformers}
Haoyu Zhao, Abhishek Panigrahi, Rong Ge, and Sanjeev Arora.
\newblock Do transformers parse while predicting the masked word?
\newblock In \emph{Proceedings of the 2023 Conference on Empirical Methods in Natural Language Processing (EMNLP)}, pages 16513--16542, Singapore, 2023. Association for Computational Linguistics.
\newblock \doi{10.18653/v1/2023.emnlp-main.1029}.
\newblock URL \url{https://aclanthology.org/2023.emnlp-main.1029/}.

\bibitem[Garnier-Brun et~al.(2025)Garnier-Brun, M{\'e}zard, Moscato, and Saglietti]{garnierbrun2024transformers}
Jerome Garnier-Brun, Marc M{\'e}zard, Emanuele Moscato, and Luca Saglietti.
\newblock How transformers learn structured data: Insights from hierarchical filtering.
\newblock In \emph{International Conference on Machine Learning (ICML)}, 2025.
\newblock arXiv:2408.15138.

\bibitem[Mossel(2016)]{mossel2016deep}
Elchanan Mossel.
\newblock Deep learning and hierarchical generative models.
\newblock \emph{arXiv preprint arXiv:1612.09057}, 2016.

\bibitem[Malach and Shalev-Shwartz(2018)]{malach2018provably}
Eran Malach and Shai Shalev-Shwartz.
\newblock A provably correct algorithm for deep learning that actually works.
\newblock \emph{arXiv preprint arXiv:1803.09522}, 2018.

\bibitem[DeGiuli(2019)]{degiuli2019random}
Eric DeGiuli.
\newblock Random language model.
\newblock \emph{Physical Review Letters}, 122\penalty0 (12):\penalty0 128301, 2019.
\newblock \doi{10.1103/PhysRevLett.122.128301}.

\bibitem[Balestriero et~al.(2023)Balestriero, Ibrahim, Sobal, Morcos, Shekhar, Goldstein, Bordes, Bardes, Mialon, Tian, Schwarzschild, Wilson, Geiping, Garrido, Fernandez, Bar, Pirsiavash, LeCun, and Goldblum]{balestriero2023cookbook}
Randall Balestriero, Mark Ibrahim, Vlad Sobal, Ari Morcos, Shashank Shekhar, Tom Goldstein, Florian Bordes, Adrien Bardes, Gregoire Mialon, Yuandong Tian, Avi Schwarzschild, Andrew~Gordon Wilson, Jonas Geiping, Quentin Garrido, Pierre Fernandez, Amir Bar, Hamed Pirsiavash, Yann LeCun, and Micah Goldblum.
\newblock A cookbook of self-supervised learning.
\newblock \emph{arXiv preprint arXiv:2304.12210}, 2023.

\bibitem[van~den Oord et~al.(2019)van~den Oord, Li, and Vinyals]{oordRepresentationLearningContrastive2019a}
Aaron van~den Oord, Yazhe Li, and Oriol Vinyals.
\newblock Representation {{Learning}} with {{Contrastive Predictive Coding}}.
\newblock \emph{arXiv preprint arXiv:1807.03748}, January 2019.
\newblock \doi{10.48550/arXiv.1807.03748}.

\bibitem[Lowe et~al.(2026)Lowe, Fuller, Oore, Shelhamer, and Taylor]{loweSelfdistillationHiddenLayers2026a}
Scott~C. Lowe, Anthony Fuller, Sageev Oore, Evan Shelhamer, and Graham~W. Taylor.
\newblock Self-distillation of hidden layers for self-supervised representation learning.
\newblock \emph{arXiv preprint arXiv:2603.15553}, March 2026.
\newblock \doi{10.48550/arXiv.2603.15553}.

\bibitem[Maes et~al.(2026)Maes, Lidec, Scieur, LeCun, and Balestriero]{maesLeWorldModelStableEndtoend2026}
Lucas Maes, Quentin~Le Lidec, Damien Scieur, Yann LeCun, and Randall Balestriero.
\newblock {{LeWorldModel}}: Stable end-to-end joint-embedding predictive architecture from pixels.
\newblock \emph{arXiv preprint arXiv:2603.19312}, March 2026.
\newblock \doi{10.48550/arXiv.2603.19312}.

\bibitem[Parley et~al.(2026)Parley, Cagnetta, and Wyart]{parley2026deep}
Jack~T Parley, Francesco Cagnetta, and Matthieu Wyart.
\newblock Deep networks learn to parse uniform-depth context-free languages from local statistics.
\newblock \emph{arXiv preprint arXiv:2602.06065}, 2026.

\bibitem[Defilippis et~al.(2026)Defilippis, Krzakala, Loureiro, and Maillard]{defilippis2026optimal}
Lorenzo Defilippis, Florent Krzakala, Bruno Loureiro, and Antoine Maillard.
\newblock Optimal scaling laws in learning hierarchical multi-index models.
\newblock \emph{arXiv preprint arXiv:2602.05846}, 2026.

\bibitem[Caron et~al.(2018)Caron, Bojanowski, Joulin, and Douze]{caronDeepClusteringUnsupervised2018}
Mathilde Caron, Piotr Bojanowski, Armand Joulin, and Matthijs Douze.
\newblock Deep clustering for unsupervised learning of visual features.
\newblock In \emph{Proceedings of the {{European Conference}} on {{Computer Vision}} ({{ECCV}})}, pages 132--149, 2018.

\bibitem[Caron et~al.(2020)Caron, Misra, Mairal, Goyal, Bojanowski, and Joulin]{caronUnsupervisedLearningVisual2020}
Mathilde Caron, Ishan Misra, Julien Mairal, Priya Goyal, Piotr Bojanowski, and Armand Joulin.
\newblock Unsupervised learning of visual features by contrasting cluster assignments.
\newblock In \emph{Proceedings of the 34th {{International Conference}} on {{Neural Information Processing Systems}}}, {{NIPS}} '20, pages 9912--9924, Red Hook, NY, USA, December 2020. Curran Associates Inc.
\newblock ISBN 978-1-7138-2954-6.

\bibitem[Oquab et~al.(2024)Oquab, Darcet, Moutakanni, Vo, Szafraniec, Khalidov, Fernandez, Haziza, Massa, {El-Nouby}, Assran, Ballas, Galuba, Howes, Huang, Li, Misra, Rabbat, Sharma, Synnaeve, Xu, Jegou, Mairal, Labatut, Joulin, and Bojanowski]{oquabDINOv2LearningRobust2024}
Maxime Oquab, Timoth{\'e}e Darcet, Th{\'e}o Moutakanni, Huy~V. Vo, Marc Szafraniec, Vasil Khalidov, Pierre Fernandez, Daniel Haziza, Francisco Massa, Alaaeldin {El-Nouby}, Mido Assran, Nicolas Ballas, Wojciech Galuba, Russell Howes, Po-Yao Huang, Shang-Wen Li, Ishan Misra, Michael Rabbat, Vasu Sharma, Gabriel Synnaeve, Hu~Xu, Herve Jegou, Julien Mairal, Patrick Labatut, Armand Joulin, and Piotr Bojanowski.
\newblock {{DINOv2}}: Learning robust visual features without supervision.
\newblock \emph{Transactions on Machine Learning Research}, January 2024.
\newblock ISSN 2835-8856.

\bibitem[Sim{\'e}oni et~al.(2026)Sim{\'e}oni, Vo, Seitzer, Baldassarre, Oquab, Jose, Khalidov, Szafraniec, Yi, Ramamonjisoa, Massa, Haziza, Wehrstedt, Wang, Darcet, Moutakanni, Sentana, Roberts, Vedaldi, Tolan, Brandt, Couprie, Mairal, Jegou, Labatut, and Bojanowski]{simeoniDINOv32026}
Oriane Sim{\'e}oni, Huy~V. Vo, Maximilian Seitzer, Federico Baldassarre, Maxime Oquab, Cijo Jose, Vasil Khalidov, Marc Szafraniec, Seung~Eun Yi, Michael Ramamonjisoa, Francisco Massa, Daniel Haziza, Luca Wehrstedt, Jianyuan Wang, Timoth{\'e}e Darcet, Th{\'e}o Moutakanni, Leonel Sentana, Claire Roberts, Andrea Vedaldi, Jamie Tolan, John Brandt, Camille Couprie, Julien Mairal, Herve Jegou, Patrick Labatut, and Piotr Bojanowski.
\newblock {{DINOv3}}.
\newblock \emph{Transactions on Machine Learning Research}, May 2026.
\newblock ISSN 2835-8856.

\bibitem[Zhou et~al.(2022)Zhou, Wei, Wang, Shen, Xie, Yuille, and Kong]{zhouIBOTImageBERT2022}
Jinghao Zhou, Chen Wei, Huiyu Wang, Wei Shen, Cihang Xie, Alan Yuille, and Tao Kong.
\newblock {{iBOT}}: Image {{BERT}} pre-training with online tokenizer.
\newblock In \emph{International {{Conference}} on {{Learning Representations}}}, January 2022.

\bibitem[Darcet et~al.(2025)Darcet, Baldassarre, Oquab, Mairal, and Bojanowski]{darcetClusterPredictLatents2025}
Timoth{\'e}e Darcet, Federico Baldassarre, Maxime Oquab, Julien Mairal, and Piotr Bojanowski.
\newblock Cluster and predict latent patches for improved masked image modeling.
\newblock \emph{Transactions on Machine Learning Research}, February 2025.
\newblock ISSN 2835-8856.

\bibitem[{Mur-Labadia} et~al.(2026){Mur-Labadia}, Muckley, Bar, Assran, Sinha, Rabbat, LeCun, Ballas, and Bardes]{mur-labadiaVJEPA21Unlocking2026}
Lorenzo {Mur-Labadia}, Matthew Muckley, Amir Bar, Mido Assran, Koustuv Sinha, Mike Rabbat, Yann LeCun, Nicolas Ballas, and Adrien Bardes.
\newblock V-{{JEPA}} 2.1: Unlocking dense features in video self-supervised learning.
\newblock \emph{arXiv preprint arXiv:2603.14482}, March 2026.
\newblock \doi{10.48550/arXiv.2603.14482}.

\bibitem[Delrocq et~al.(2026)Delrocq, Zihan, Bellec, and Gerstner]{delrocqSelfsupervisedLocalLearning2026}
Ariane Delrocq, Wu~S. Zihan, Guillaume Bellec, and Wulfram Gerstner.
\newblock Self-supervised local learning rules learn the hidden hierarchical structure of high-dimensional data.
\newblock \emph{arXiv preprint arXiv:2605.18557}, May 2026.
\newblock \doi{10.48550/arXiv.2605.18557}.

\bibitem[Quinton and Rey(2025)]{quintonJacobianDescentMultiobjective2025}
Pierre Quinton and Val{\'e}rian Rey.
\newblock Jacobian descent for multi-objective optimization.
\newblock \emph{arXiv preprint arXiv:2406.16232}, February 2025.
\newblock \doi{10.48550/arXiv.2406.16232}.

\bibitem[Akiba et~al.(2019)Akiba, Sano, Yanase, Ohta, and Koyama]{akibaOptunaNextgenerationHyperparameter2019}
Takuya Akiba, Shotaro Sano, Toshihiko Yanase, Takeru Ohta, and Masanori Koyama.
\newblock Optuna: A next-generation hyperparameter optimization framework.
\newblock In \emph{Proceedings of the 25th {{ACM SIGKDD International Conference}} on {{Knowledge Discovery}} \& {{Data Mining}}}, {{KDD}} '19, pages 2623--2631, New York, NY, USA, July 2019. Association for Computing Machinery.
\newblock ISBN 978-1-4503-6201-6.
\newblock \doi{10.1145/3292500.3330701}.

\end{thebibliography}

\newpage 
\appendix

\section{Further related work\label{sec:app:further_related_works}}

\paragraph{Approximation- and information-theoretic bounds.}
Deep networks can represent hierarchically compositional functions with exponentially fewer parameters than shallow ones~\citep{poggio2017why}, which translates into nonparametric sample-complexity bounds polynomial in the input dimension for both supervised~\citep{schmidt2020nonparametric} and self-supervised~\citep{mei2024unets} tasks. These results characterize expressivity and information-theoretic statistics, but do not describe what gradient descent can actually learn from a finite training set.

\paragraph{Supervised learning.}
\citet{mossel2016deep} introduced phylogeny-inspired generative models on regular trees, and \citet{malach2018provably} showed that they can be learned by an iterative clustering algorithm. \citet{degiuli2019random} introduced languages with random production rules. The RHM~\citep{cagnetta2024deep} combines both assumptions, leading to models with controllable correlations where sample complexity can be predicted and compared to observations in deep nets. \citet{ren2026provable} proved rigorously that result for multi-step gradient descent. \citet{parley2026deep} obtains sample complexity for non-regular trees. \citet{defilippis2026optimal} derive optimal scaling laws for hierarchical multi-index models, a related class of generative models with continuous latents.

\paragraph{Self-supervised learning.}
A complementary line of work has shown that transformers can approximately implement parsing-style (inside-algorithm) inference on context-free grammars~\citep{allenzhu2024cfg, zhao2023transformers, garnierbrun2024transformers}. These results interpret the algorithm that trained LLMs implement, but do not characterize the learning mechanism nor its sample complexity. Such analyses do exist for the RHM, but only for next-token prediction~\citep{cagnetta2024deep} and for diffusion~\citep{favero2025compositional,favero2025bigger}, not for learning from one's own latents as considered here.

\paragraph{Clustering-based approaches to self-distillation.}
Our SLC method is similar in spirit to clustering or self-learned pseudo-labeling-based methods for image-representation learning, such as DeepCluster \cite{caronDeepClusteringUnsupervised2018} which in turn evolved into SwAV \cite{caronUnsupervisedLearningVisual2020}, the DINO family of models \cite{caronEmergingPropertiesSelfsupervised2021, oquabDINOv2LearningRobust2024, simeoniDINOv32026}, iBOT \cite{zhouIBOTImageBERT2022}, and CAPI \cite{darcetClusterPredictLatents2025}. Our method adds \textit{explicit} hierarchical clustering -- it would be interesting to test to what degree \textit{implicit} hierarchical supervision is present for these clustering-based approaches. 

\paragraph{Explicitly hierarchical self-distillation SSL.} 
Schematic realizations of the H-JEPA \cite{lecun2022path}, V-JEPA 2.1 \cite{mur-labadiaVJEPA21Unlocking2026}, Bootleg \cite{loweSelfdistillationHiddenLayers2026a}, and our SLC architectures are depicted in \cref{fig:hjepa_bootleg_schematic}. Bootleg and V-JEPA 2.1 are extremely recent, report state of the art performance, and employ multi-loss optimization with one prediction per sampled representation level. Bootleg is schematically very similar to data2vec (cf. \cref{fig:d2vec_online_class}--\textbf{left}), except that in bootleg gradients are computed after averaging losses over representations, whereas in data2vec the averaging is done first over representations. V-JEPA 2.1 meanwhile collates representations from throughout the student network, similar to H-JEPA, but also allows low-level target predictions to be influenced by higher level latents. This high$\rightarrow$low prediction path is missing in H-JEPA, but present in SLC, V-JEPA 2.1, and Bootleg. It is important to clarify whether this is related to the fact that, with few exceptions~\cite{liHiTJEPAHierarchicalSelfsupervised2025,girgisHierarchicalJEPAMeets2026}, ``standard'' H-JEPA appears to have found limited application in the literature. 

\paragraph{Predictive coding and the RHM.}
\citet{delrocqSelfsupervisedLocalLearning2026} recently showed that a backprop-free predictive coding model, CLAPP, can learn the RHM's hierarchical structure with sample complexity comparable to that of backprop. In that work, the model was trained on the RHM in the setting where all production rules exist, such that all sentences belong to the language (corresponding to $f=1$ in our notations). In that case, self-supervised methods cannot learn the RHM rules, as all inputs are equiprobable and correlations within the input are absent. However, the variation of CLAPP considered in that work uses a contrastive update rule that is gated by an extra learning signal: updates are given a $\pm 1$ sign indicating whether pairs  belong to the same top-level class. The sample complexity obtained was that of supervised learning $\sim m^L$. Baseline CLAPP  \cite{illingLocalPlasticityRules2021} does not rely on this non-local context supervision -- it would be interesting to understand whether it learns the RHM rules in sample complexity $m^3$ when $f<1$. 

\clearpage

\begin{figure}[p]
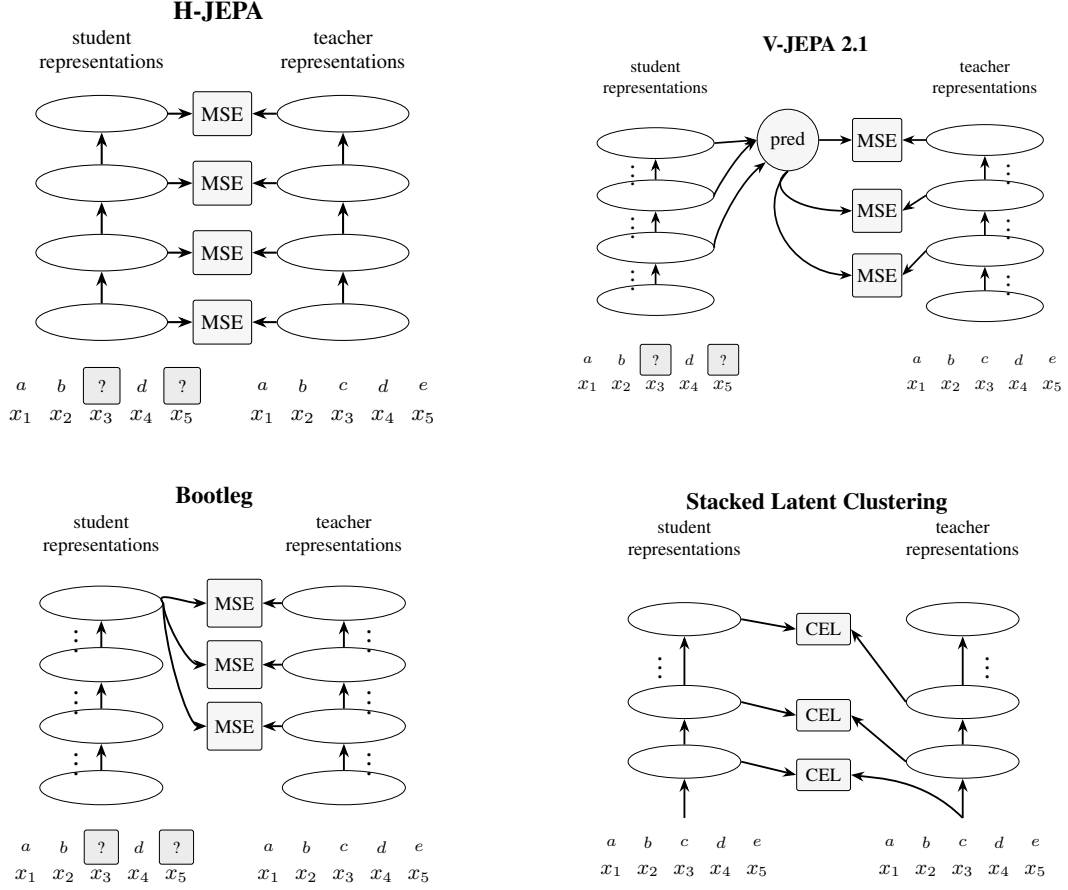

    \centering
    \vspace*{\fill}
    \resizebox{\linewidth}{!}{\begingroup
\providecommand{\sslHjepaPanelX}{}
\providecommand{\sslHjepaPanelY}{}
\providecommand{\sslVjepaPanelX}{}
\providecommand{\sslVjepaPanelY}{}
\providecommand{\sslBootlegPanelX}{}
\providecommand{\sslBootlegPanelY}{}
\providecommand{\sslNeuralClusteringPanelX}{}
\providecommand{\sslNeuralClusteringPanelY}{}
\providecommand{\sslHjepaPanelWidth}{}
\providecommand{\sslVjepaPanelWidth}{}
\providecommand{\sslBootlegPanelWidth}{}
\providecommand{\sslNeuralClusteringPanelWidth}{}
\providecommand{\schematicSslBoundingBoxPath}{}
\renewcommand{\sslHjepaPanelX}{-3.65}
\renewcommand{\sslHjepaPanelY}{2.90}
\renewcommand{\sslVjepaPanelX}{3.65}
\renewcommand{\sslVjepaPanelY}{2.90}
\renewcommand{\sslBootlegPanelX}{-3.65}
\renewcommand{\sslBootlegPanelY}{-2.90}
\renewcommand{\sslNeuralClusteringPanelX}{3.65}
\renewcommand{\sslNeuralClusteringPanelY}{-2.90}
\renewcommand{\sslHjepaPanelWidth}{5.20cm}
\renewcommand{\sslVjepaPanelWidth}{5.85cm}
\renewcommand{\sslBootlegPanelWidth}{4.90cm}
\renewcommand{\sslNeuralClusteringPanelWidth}{5.85cm}
\renewcommand{\schematicSslBoundingBoxPath}{\path[use as bounding box] (-4.10,-5.50) rectangle (6.50,5.50);}

\begin{tikzpicture}
    \schematicSslBoundingBoxPath
    \node[inner sep=0pt] at (\sslHjepaPanelX,\sslHjepaPanelY) {%
        \resizebox{\sslHjepaPanelWidth}{!}{%
            \IfFileExists{figs/schematic_ssl_hjepa.tikz}{\input{figs/schematic_ssl_hjepa.tikz}}{\input{schematic_ssl_hjepa.tikz}}%
        }%
    };
    \node[inner sep=0pt] at (\sslVjepaPanelX,\sslVjepaPanelY) {%
        \resizebox{\sslVjepaPanelWidth}{!}{%
            \IfFileExists{figs/schematic_ssl_vjepa21.tikz}{\input{figs/schematic_ssl_vjepa21.tikz}}{\input{schematic_ssl_vjepa21.tikz}}%
        }%
    };
    \node[inner sep=0pt] at (\sslBootlegPanelX,\sslBootlegPanelY) {%
        \resizebox{\sslBootlegPanelWidth}{!}{%
            \IfFileExists{figs/schematic_ssl_bootleg.tikz}{\input{figs/schematic_ssl_bootleg.tikz}}{\input{schematic_ssl_bootleg.tikz}}%
        }%
    };
    \node[inner sep=0pt] at (\sslNeuralClusteringPanelX,\sslNeuralClusteringPanelY) {%
        \resizebox{\sslNeuralClusteringPanelWidth}{!}{%
            \IfFileExists{figs/schematic_ssl_neural_clustering.tikz}{\input{figs/schematic_ssl_neural_clustering.tikz}}{\input{schematic_ssl_neural_clustering.tikz}}%
        }%
    };
\end{tikzpicture}
\endgroup}
    \caption{\textbf{Top left:} Structure of losses and targets for na\"ive H-JEPAs. \textbf{Top right:} Structure of losses and targets in V-JEPA 2.1. \textbf{Bottom left:} Structure of losses and targets for Bootleg. \textbf{Bottom right:} A schematic implementation of Stacked Latent Clustering, in which high-level representations are grounded in predicted lower-level representations.}
    \label{fig:hjepa_bootleg_schematic}
    \vspace*{\fill}
\end{figure}

\clearpage

\section{Proof of Iterative Latent Clustering recovery}
Here, we provide formal definitions for the random hierarchy model and of the ILC algorithm. We then prove the recovery guarantee for ILC, which is stated informally in \Cref{sec:latent-clustering}. 

\subsection{The random hierarchy model \label{sec:rhm_formal}}
The RHM \citep{cagnetta2024deep} is a probabilistic context-free grammar on a fixed regular tree. The tree has depth \(L\), branching factor \(s\), and vocabularies \(\mathcal V_0,\mathcal V_1,\ldots,\mathcal V_L\), all of size \(v\). Level \(0\) is visible and levels \(1,\ldots,L\) are latent. At level \(\ell\), there are \(s^{L-\ell}\) variables \(h^{(\ell)}_1,\ldots,h^{(\ell)}_{s^{L-\ell}}\), with \(h^{(\ell)}_u\in\mathcal V_\ell\). The visible sequence is \(x_i=h^{(0)}_i\), \(i=1,\ldots,s^L\).

For each level \(\ell=0,\ldots,L-1\), a grammar instance is sampled by choosing \(vm\) distinct tuples from \(\mathcal V_\ell^s\) uniformly without replacement and partitioning them into \(v\) labeled groups of size \(m\), one group \(\mathcal R_{\ell,a}\) for each parent symbol \(a\in\mathcal V_{\ell+1}\). A rule \(\nu=(\nu_1,\ldots,\nu_s)\in\mathcal R_{\ell,a}\) means that parent \(a\) may generate the tuple of children \(\nu\). The set of grammatical tuples at level \(\ell\) is
\(
    \mathcal S_\ell:=\bigsqcup_{a\in\mathcal V_{\ell+1}}\mathcal R_{\ell,a},
\)
so \(|\mathcal S_\ell|=vm\). We write \(f:=m/v^{s-1}\) for the fraction of valid \(s\)-tuples, since \(vm\) out of the \(v^s\) possible tuples are grammatical. Because the rule map is injective, the grammar is unambiguous: each grammatical tuple has a unique parent. We denote it by
\(
    \mathrm{par}_\ell(\nu)=a
    \;\Longleftrightarrow\;
    \nu\in\mathcal R_{\ell,a}.
\)
Two grammatical tuples \(\nu,\nu'\in\mathcal S_\ell\) are called \textit{synonyms} if they have the same parent.
The map \(\mathrm{par}_\ell\) is the map \(h_\ell\) used in the main text. We write \(\Phi_\ell(a)\) for the parent context vector associated with parent value \(a\in\mathcal V_{\ell+1}\), so that \(\phi_\ell(\nu)=\Phi_\ell(\mathrm{par}_\ell(\nu))\).

Generation proceeds top-down. First \(h^{(L)}_1\sim\mathrm{Unif}(\mathcal V_L)\). Then, recursively, if \(h^{(\ell+1)}_u=a\), the child tuple
\(
    T^{(\ell)}_u:=\big(h^{(\ell)}_{(u-1)s+1},\ldots,h^{(\ell)}_{us}\big)
\)
is sampled uniformly from \(\mathcal R_{\ell,a}\).

\subsection{Explicit clustering algorithm \label{sec:ilc_alg_formal}}
Below, we present the formal ILC clustering algorithm, initially presented schematically in \cref{fig:ilc_alg}. 
\begin{algorithm}[t]
\DontPrintSemicolon
\newcommand{\ilctikzmark}[1]{{\tikz[overlay,remember picture,baseline]\node (#1) {};}}
\newcommand{\ilcwrap}[1]{\parbox[t]{0.76\linewidth}{#1}}
\newcommand*{\AddILCNote}[4]{%
    \begin{tikzpicture}[overlay,remember picture]
        \draw[decoration={brace,amplitude=1.15em},decorate,thick]
            ($(#3)!(#1.north)!($(#3)-(0,1)$) + (1.15em,0)$) --
            ($(#3)!(#2.south)!($(#3)-(0,1)$) + (1.15em,0)$)
            node[align=left,pos=0.5,anchor=west,xshift=1.15em] {#4};
    \end{tikzpicture}%
}
\newcommand*{\PlaceILCNotes}{%
    \makebox[0pt][l]{%
    \raisebox{0pt}[0pt][0pt]{%
        \AddILCNote{ilc-predictor-top}{ilc-predictor-bottom}{ilc-brace-right}{Predictor, $p$}%
        \AddILCNote{ilc-clusterer-top}{ilc-clusterer-bottom}{ilc-brace-right}{Clusterer, $C$}%
    }}%
}
\caption{Iterative Latent Clustering (ILC) --- see \Cref{fig:ilc_schematic} for a graphical representation.}
\label{alg:ilc}
\KwIn{$P$ samples $x^{(1)},\ldots,x^{(P)}$, RHM parameters $L,s,v$, and a clustering module $\mathsf{Cluster}_v$.}
\KwOut{estimated non-root hierarchy $\widehat h^{(1)},\ldots,\widehat h^{(L-1)}$.}
\BlankLine
\textbf{Initialize:} $\widehat h^{(0)}_i = x_i$.\;
\BlankLine
\For{$\ell = 0, 1, \ldots, L-2$}{
  \ilcwrap{\ilctikzmark{ilc-predictor-top}Form all level-$\ell$ tuples $\widehat T^{(\ell)}_u$.\hfill\ilctikzmark{ilc-brace-right}}\;
  \ilcwrap{Estimate the support of grammatical tuples $\widehat{\mathcal S}_\ell$ by the observed tuple set.}\;
  \ilcwrap{For every $\nu\in\widehat{\mathcal S}_\ell$, count its incidence as $N(\nu) = \sum_{p=1}^P
          \mathbf 1\{
              \widehat T^{(p)}_\ell=\nu
          \}$ and compute its empirical cousin context vector}
  \[
      \widehat\phi_\ell(\nu)
      :=
      \frac{1}{N(\nu)}
          \sum_{p=1}^P
          e_{\widehat Z^{(p)}_\ell}
          \mathbf 1\{
              \widehat T^{(p)}_\ell=\nu
          \}
      ,
  \]
  \ilcwrap{where $\widehat T^{(p)}_\ell$ is a fixed tuple at level $\ell$ in sample $p$, and $\widehat Z^{(p)}_\ell$ is a fixed level-$\ell$ element in a cousin tuple (i.e. sharing the $\ell+2$ grandparent with $\widehat T^{(p)}_\ell$).\ilctikzmark{ilc-predictor-bottom}}\;
  \ilcwrap{\ilctikzmark{ilc-clusterer-top}Cluster the context vectors:
  $\widehat{\mathcal S}_{\ell,1},\ldots,\widehat{\mathcal S}_{\ell,v}
  = \mathsf{Cluster}_v\!\left(\{\widehat\phi_\ell(\nu):\nu\in\widehat{\mathcal S}_\ell\}\right).$}\;
  \ilcwrap{Define the next-level latent label by the cluster identity:}
  \[
      \widehat h^{(\ell+1)}_u=a
      \qquad
      \text{if}
      \qquad
      \widehat T^{(\ell)}_u
      \in
      \widehat{\mathcal S}_{\ell,a}.\ilctikzmark{ilc-clusterer-bottom}
  \]\PlaceILCNotes \vspace{-0.75em}
}
\end{algorithm}

\begin{wrapfigure}{br}{0.54\textwidth}
    \vspace{1em}
    \centering 
    \resizebox{\linewidth}{!}{\begingroup
\providecommand{\clusteringSchematicPanelSep}{}
\providecommand{\clusteringSchematicXSep}{}
\providecommand{\clusteringSchematicLevelOneY}{}
\providecommand{\clusteringSchematicLevelTwoY}{}
\providecommand{\clusteringSchematicRootY}{}
\renewcommand{\clusteringSchematicPanelSep}{8mm}
\renewcommand{\clusteringSchematicXSep}{0.80}
\renewcommand{\clusteringSchematicLevelOneY}{-2.3}
\providecommand{\clusteringSchematicPDeltaY}{1.3}
\providecommand{\clusteringSchematicCDeltaY}{0.7}
\renewcommand{\clusteringSchematicLevelTwoY}{-0.25}
\renewcommand{\clusteringSchematicRootY}{0.6}
\definecolor{ptcOchre}{RGB}{184,119,0}
\definecolor{ptcTeal}{RGB}{0,128,128}
\definecolor{ptcViolet}{RGB}{118,80,160}
\definecolor{ptcBlue}{RGB}{42,111,180}
\definecolor{ptcRed}{RGB}{190,72,72}
\definecolor{rhmTeal}{RGB}{0,128,128}
\definecolor{rhmViolet}{RGB}{118,80,160}
\providecommand{\clusteringSchematicXSep}{0.80}
\providecommand{\clusteringSchematicLevelOneY}{-1.20}
\providecommand{\clusteringSchematicLevelTwoY}{1.75}
\providecommand{\clusteringSchematicRootY}{2.12}
\begin{tikzpicture}[
    >=Stealth,
    node dot/.style={circle, draw, thick, fill=white, inner sep=0pt, minimum size=5.6mm},
    leaf dot/.style={circle, draw, thick, fill=white, inner sep=0pt, minimum size=4.6mm},
    module/.style={regular polygon, regular polygon sides=3, shape border rotate=0, draw, thick, fill=white, inner sep=0pt, minimum size=7.2mm, font=\footnotesize},
    target/.style={draw, thick, fill=white, minimum width=13mm, minimum height=6.4mm, inner sep=0pt},
    latent/.style={circle, draw, thick, fill=white, inner sep=0pt, minimum size=3.2mm},
    tree edge/.style={line width=0.8pt},
    signal/.style={->, line width=1.1pt}
]
    \path[use as bounding box] (-2,-4.55) rectangle (3.85,0.0);
    \IfFileExists{figs/clustering_schematic_ptc_module_panel.tikz}{\input{figs/clustering_schematic_ptc_module_panel.tikz}}{\input{clustering_schematic_ptc_module_panel.tikz}}
\end{tikzpicture}
\endgroup}
    \caption{\textbf{Graphical representation of \Cref{alg:ilc} for $L=3$}. Predictor $p$ implements steps 3-5, and clusterer $C$ constructs next-level latents. Highlighted prediction targets are as in \Cref{fig:rhm_corr_schematic}d. \label{fig:ilc_schematic}}
\end{wrapfigure}

\subsection{Proof of recovery guarantee \label{sec:ilc_proof}}
\label{app:clustering-proof}

This appendix proves the recovery guarantee for Iterative Latent Clustering. Throughout the proof, probabilities are taken over fresh samples from a fixed grammar instance, unless an expectation over grammars is explicitly denoted by \(\mathbb E_{\mathcal G}\).

\subsubsection{Formal assumptions}

We use the following two assumptions. The first is a typical-event assumption on the random grammar. The second is an algorithmic stability assumption on the clustering module.

\begin{assumption}[Balanced and separated grammar]
\label{ass:app-balanced-separated}
There exist constants \(c_{\rm bal},C_{\rm bal},c_{\rm sep}>0\), independent of \(v,m,L\), such that for every level \(\ell=0,\ldots,L-2\), every grammatical tuple \(\nu\in\mathcal S_\ell\) satisfies
\[
    \frac{c_{\rm bal}}{vm}\leq \mathbb P[T=\nu]\leq \frac{C_{\rm bal}}{vm},
\]
and the parent context vectors satisfy
\[
    \min_{a\neq b}\|\Phi_\ell(a)-\Phi_\ell(b)\|_2\geq c_{\rm sep}\frac{\sqrt{1-f}}{m}.
\]
\end{assumption}

\begin{assumption}[Stable clustering module]
\label{ass:app-stable-clustering}
Let \(\{z_\nu:\nu\in\mathcal S_\ell\}\subset\mathbb R^v\) be points with true centers \(\{\Phi_\ell(a):a\in\mathcal V_{\ell+1}\}\). Suppose that
\[
    \|z_\nu-\Phi_\ell(\mathrm{par}_\ell(\nu))\|_2\leq \varepsilon
\]
for all \(\nu\in\mathcal S_\ell\), and that
\[
    \|\Phi_\ell(a)-\Phi_\ell(b)\|_2\geq \Delta
\]
for all \(a\neq b\). If \(\varepsilon\leq \Delta/8\), then \(\mathsf{Cluster}_v(\{z_\nu\}_{\nu\in\mathcal S_\ell})\) returns the true partition \(\mathcal S_\ell=\bigsqcup_{a\in\mathcal V_{\ell+1}}\mathcal S_{\ell,a}\), up to a permutation of labels.
\end{assumption}

Assumption~\ref{ass:app-balanced-separated} is a high-probability random-grammar event. Balancedness follows from asymptotic uniformity of single-symbol marginals in the limit \(v\to\infty\) at fixed \(f\in(0,1)\), see Lemma C.1 of \cite{parley2026deep}. Indeed, if \(\nu\in\mathcal S_{\ell,a}\), then \(\nu\) is generated with probability \(1/m\) conditional on its parent \(a\). Therefore
\[
    \mathbb P[T^{(\ell)}=\nu]
    =
    \frac1m\mathbb P[h^{(\ell+1)}=a]
    \sim
    \frac1{vm}.
\]
The separation scale is the latent-latent correlation scale. Indeed, writing
\[
    \psi_\ell(\nu):=\phi_\ell(\nu)-\mathbb E[e_Y],
\]
the connected covariance column is
\[
    C_\ell(:,\nu)=\mathbb P[T=\nu]\psi_\ell(\nu).
\]
For \(\ell=0\), \(T\) is a grammatical visible \(s\)-tuple and \(Y\) is a visible token in a sibling branch, with lowest common ancestor two levels above the leaves. This is precisely the token--tuple covariance computed in \citet{cagnetta2024towards}. In our notation, their calculation gives, for a fixed visible value \(y\) and a grammatical tuple \(\nu\),
\[
    \mathbb E_{\mathcal G}\!\left[C_0(y,\nu)^2\,\middle|\,\nu\in\mathcal S_0\right]\asymp \frac{1-f}{v^3m^4}.
\]
Summing over the \(v\) possible values of \(Y\) yields the squared column norm
\[
    \mathbb E_{\mathcal G}\!\left[\|C_0(:,\nu)\|_2^2\,\middle|\,\nu\in\mathcal S_0\right]\asymp v\cdot \frac{1-f}{v^3m^4}=\frac{1-f}{v^2m^4}.
\]

The same estimate holds at every latent level by self-similarity of the RHM. Indeed, fix \(\ell>0\) and consider the sequence of level-\(\ell\) variables
\[
    H^{(\ell)}:=\big(h^{(\ell)}_1,\ldots,h^{(\ell)}_{s^{L-\ell}}\big).
\]
Marginally, \(H^{(\ell)}\) is generated by an RHM of depth \(L-\ell\) with the same local parameters \((s,v,m)\), using the production rules at levels \(\ell,\ell+1,\ldots,L-1\). In this truncated RHM, the level-\(\ell\) variables play the role of visible tokens. Our statistic \(C_\ell(:,\nu)\) is therefore exactly the level-\(0\) token--tuple covariance of this truncated model: \(T\) is a terminal \(s\)-tuple, \(Y\) is a terminal child in a sibling branch, and their lowest common ancestor is again two levels above the terminal scale. Since the random rule ensemble has the same law at every level, the same covariance calculation gives
\[
    \mathbb E_{\mathcal G}\!\left[\|C_\ell(:,\nu)\|_2^2\,\middle|\,\nu\in\mathcal S_\ell\right]\asymp \frac{1-f}{v^2m^4},
\]
uniformly for \(\ell=0,\ldots,L-2\).
Since
\[
    C_\ell(:,\nu)=\mathbb P[T=\nu]\psi_\ell(\nu)
\]
and, on the balancedness event, \(\mathbb P[T=\nu]\asymp 1/(vm)\), this implies
\[
    \mathbb E_{\mathcal G}\!\left[\|\psi_\ell(\nu)\|_2^2\,\middle|\,\nu\in\mathcal S_\ell\right]\asymp \frac{(1-f)/(v^2m^4)}{1/(v^2m^2)}=\frac{1-f}{m^2}.
\]

Thus \(\sqrt{1-f}/m\) is the natural scale of the centered context-vector signal. 

To turn this into a separation statement, we need to assume that different parent context vectors have generic relative positions, as expected for independent random production rules. In particular, under a random-direction heuristic, overlaps are smaller than squared norms by a factor \(v^{-1/2}\), so pairwise distances have the same scale as the individual norms.

\subsubsection{Theorem and proof}

\begin{theorem}[Recovery of the non-root hierarchy]
\label{thm:iterative-latent-clustering-formal}
Assume the balancedness and separation conditions above, and assume the clustering is stable under perturbations as described. Then there exists a constant \(C>0\), depending only on the constants in those assumptions and on \(s\), such that the following holds. If
\begin{equation}
    P
    \geq
    C
    \left[
        vm\log\frac{Lvm}{\delta}
        +
        \frac{v m^3}{1-f}
        \log\frac{Lvm}{\delta}
    \right],
\end{equation}
then the iterative latent-clustering algorithm recovers
\(
    h^{(1)},h^{(2)},\ldots,h^{(L-1)}
\)
and all production-rule classes below the root with probability at least \(1-\delta\), up to independent permutations of the latent labels at each level.

In particular, up to logarithmic factors,
\(
    P_{\mathrm{ILC}} =
    \widetilde O
    \left(
    (1-f)^{-1} \, v m^3
    \right).
\)
\end{theorem}

\paragraph{Synonyms have identical context vectors.}

We first record the property that makes clustering possible.

\begin{lemma}[Synonyms have identical context vectors]
\label{lem:app-synonyms}
If \(\nu,\nu'\in\mathcal S_\ell\) satisfy \(\mathrm{par}_\ell(\nu)=\mathrm{par}_\ell(\nu')\), then \(\phi_\ell(\nu)=\phi_\ell(\nu')\).
\end{lemma}

\begin{proof}
Let \(A\) be the level-\((\ell+1)\) parent of the branch containing \(T\), let \(B\) be the level-\((\ell+1)\) parent of the sibling branch containing \(Y\), and let \(G\) be their shared level-\((\ell+2)\) parent. The local graphical structure is \(G\to(A,B)\), \(A\to T\), and \(B\to Y\). By unambiguity, observing \(T=\nu\) determines \(A=\mathrm{par}_\ell(\nu)\). By the context-free Markov structure, conditioned on \(A\), the rest of the tree is independent of which production rule was used to generate \(T\). Hence, if \(\nu\in\mathcal S_{\ell,a}\),
\[
    \mathbb P[Y=y\mid T=\nu]=\mathbb P[Y=y\mid A=a].
\]
The right-hand side depends on \(\nu\) only through its parent \(a\). Therefore all tuples in the same synonym class have the same context vector.
\end{proof}

\paragraph{Concentration of empirical context vectors.}

Fix a level \(\ell\) and a grammatical tuple \(\nu\in\mathcal S_\ell\). Let \(N_\nu\) be the number of samples in which the chosen cousin tuple equals \(\nu\), and let \(Y_1,\ldots,Y_{N_\nu}\) be the corresponding target values. Conditional on \(T=\nu\), these are independent draws from the conditional law of \(Y\), and the empirical context vector
\[
    \widehat\phi_\ell(\nu)=\frac1{N_\nu}\sum_{q=1}^{N_\nu} e_{Y_q}
\]
is the empirical mean of \(N_\nu\) independent one-hot vectors with mean \(\phi_\ell(\nu)=\mathbb E[e_Y\mid T=\nu]\).

\begin{lemma}[Context-vector concentration]
\label{lem:app-context-concentration}
There is a constant \(C_0>0\) such that, for any \(\eta\in(0,e^{-1}]\),
\[
    \mathbb P\!\left[\|\widehat\phi_\ell(\nu)-\phi_\ell(\nu)\|_2>C_0\sqrt{\frac{\log(1/\eta)}{N_\nu}}\,\middle|\,N_\nu\right]\leq \eta.
\]
\end{lemma}

\begin{proof}
Condition on \(N_\nu=n\geq1\).
Define
\[
    U_q:=e_{Y_q}-\phi_\ell(\nu),
    \qquad
    \bar U:=\frac1n\sum_{q=1}^n U_q.
\]
Then \(\mathbb E[U_q]=0\) and
\[
    \widehat\phi_\ell(\nu)-\phi_\ell(\nu)=\bar U.
\]
Since the \(U_q\)'s are independent and mean zero, the cross terms vanish, and
\[
\begin{aligned}
    \mathbb E\|\bar U\|_2^2
    &=
    \mathbb E\left\|
        \frac1n\sum_{q=1}^n U_q
    \right\|_2^2 \\
    &=
    \frac1{n^2}\sum_{q=1}^n \mathbb E\|U_q\|_2^2 \\
    &\leq \frac1n,
\end{aligned}
\]
where we used
\[
    \mathbb E\|e_Y-\phi_\ell(\nu)\|_2^2
    =
    1-\|\phi_\ell(\nu)\|_2^2
    \leq 1.
\]
Thus,
\[
    \mathbb E\|\bar U\|_2\leq n^{-1/2}.
\]

Now consider the function
\[
    F(Y_1,\ldots,Y_n):=\|\bar U\|_2.
\]
Changing one observation changes \(\bar U\) by at most \(\sqrt2/n\), and hence changes \(F\) by at most \(\sqrt2/n\). By McDiarmid's bounded-differences inequality,
\[
    \mathbb P\!\left[
        F-\mathbb EF>t
    \right]
    \leq
    \exp(-n t^2).
\]
Taking
\[
    t=\sqrt{\frac{\log(1/\eta)}{n}},
\]
we get, with probability at least \(1-\eta\),
\[
    F
    \leq
    \frac1{\sqrt n}
    +
    \sqrt{\frac{\log(1/\eta)}{n}}.
\]
Since \(\eta\leq e^{-1}\), we have \(\log(1/\eta)\geq1\), so the first term is absorbed into the second. Thus, for a constant \(C_0\),
\[
    \|\widehat\phi_\ell(\nu)-\phi_\ell(\nu)\|_2
    =
    \|\bar U\|_2
    \leq
    C_0\sqrt{\frac{\log(1/\eta)}{n}}
\]
with probability at least \(1-\eta\).
\end{proof}

\paragraph{Proof of the theorem.}

At each level \(\ell\), let \(T_\ell\) be the level-\(\ell\)
tuple to be clustered and \(Z_\ell\) the level-\(\ell\) target in the
sibling branch. For sample \(p\), denote the corresponding true variables
by \(T_\ell^{(p)}\) and \(Z_\ell^{(p)}\)

We now prove \Cref{thm:iterative-latent-clustering-formal}.

\begin{proof}[Proof of \Cref{thm:iterative-latent-clustering-formal}]
We first define a good event using the true latent variables. For every level \(\ell=0,\ldots,L-2\) and every grammatical tuple \(\nu\in\mathcal S_\ell\), let
\[
    N_{\ell,\nu}:=\sum_{p=1}^P \mathbf 1\{T_\ell^{(p)}=\nu\}
\]
be the number of samples in which the true level-\(\ell\) cousin tuple equals \(\nu\), and let
\[
    \widetilde\phi_\ell(\nu):=
    \frac{1}{N_{\ell,\nu}}\sum_{p=1}^P e_{Z_\ell^{(p)}}\mathbf 1\{T_\ell^{(p)}=\nu\}
\]
be the corresponding oracle empirical context vector.

By balancedness,
\[
    \mathbb P[T_\ell=\nu]\geq \frac{c_{\rm bal}}{vm}.
\]
Thus \(N_{\ell,\nu}\) is binomial with mean at least \(c_{\rm bal}P/(vm)\). A Chernoff bound and a union bound over all \((L-1)vm\) pairs \((\ell,\nu)\) imply that, if
\[
    P\gtrsim vm\log\frac{Lvm}{\delta},
\]
then, with probability at least \(1-\delta/2\),
\[
    N_{\ell,\nu}\gtrsim \frac{P}{vm}
\]
simultaneously for all \(\ell\) and \(\nu\).

On this count event, Lemma~\ref{lem:app-context-concentration}, applied with \(\eta=\delta/(2Lvm)\), and another union bound give
\[
    \max_{\ell,\nu}
    \|\widetilde\phi_\ell(\nu)-\phi_\ell(\nu)\|_2
    \lesssim
    \sqrt{\frac{vm}{P}\log\frac{Lvm}{\delta}}
\]
with probability at least \(1-\delta/2\). Therefore, if
\[
    P\gtrsim \frac{vm^3}{1-f}\log\frac{Lvm}{\delta},
\]
then, with probability at least \(1-\delta\),
\[
    \|\widetilde\phi_\ell(\nu)-\phi_\ell(\nu)\|_2
    \leq
    \frac18 c_{\rm sep}\frac{\sqrt{1-f}}{m}
\]
for all \(\ell\) and all \(\nu\in\mathcal S_\ell\). Call this simultaneous event \(\mathcal E\).

We now show that on \(\mathcal E\), the algorithm succeeds deterministically. At level \(0\), the variables are visible tokens, so \(\widehat h^{(0)}=h^{(0)}\). Suppose inductively that level \(\ell\) has been recovered exactly, up to a permutation of labels. Then the recovered level-\(\ell\) tuples are the true level-\(\ell\) tuples, up to relabeling. Consequently, the empirical context vectors computed by the algorithm are exactly the oracle vectors \(\widetilde\phi_\ell(\nu)\), up to permutations of tuple labels and output coordinates. These permutations preserve distances.

Thus, on \(\mathcal E\), every empirical context vector lies within
\[
    \frac18 c_{\rm sep}\frac{\sqrt{1-f}}{m}
\]
of its true parent center. By separation,
\[
    \min_{a\neq b}\|\Phi_\ell(a)-\Phi_\ell(b)\|_2
    \geq
    c_{\rm sep}\frac{\sqrt{1-f}}{m}.
\]
The stable clustering assumption therefore implies that the clustering step recovers the true synonym partition
\[
    \mathcal S_\ell
    =
    \bigsqcup_{a\in\mathcal V_{\ell+1}}\mathcal S_{\ell,a}
\]
up to a permutation of labels. Assigning each tuple to its cluster recovers level \(\ell+1\), again up to a permutation of labels.

Iterating this deterministic induction from \(\ell=0\) to \(L-2\) recovers \(h^{(1)},\ldots,h^{(L-1)}\) and all production-rule classes below the root. Since \(\mathcal E\) holds with probability at least \(1-\delta\), the theorem follows.

The required sample size is
\[
    P\gtrsim vm\log\frac{Lvm}{\delta}
    +
    \frac{vm^3}{1-f}\log\frac{Lvm}{\delta}.
\]
Thus, up to logarithmic factors,
\(P \gtrsim vm^3/(1-f)\) suffices for recovery.
\end{proof}

\section{Stacked latent-clustering details and additional results\label{sec:app:slc}}

This appendix gives implementation details and additional diagnostics for the SLC network introduced in \Cref{sec:neuralclustering}. We use the RHM notation of \Cref{sec:rhm}: true latent variables are denoted by $h^{(\ell)}_i$, while learned SLC tokens are denoted by $\widehat h^{(\ell)}_i$.

\subsection{SLC architecture details \label{sec:app:slc_arch_details}}
We encode RHM input strings of size $s^L$ into one-hot feature vectors of dimension $\mathbb{R}^{s^L \times d_h}$, where $d_h \gg v$ is the dimension into which we tokenize learned latents. This initializes the level-$0$ tokens as $\widehat h^{(0)}_i=x_i$.
We use a single dimension $d_h$ for the cluster codebook and for the token vocabulary passed between SLC modules: each clusterer assigns labels in $\{1,\ldots,d_h\}$, these labels become the tokens for the next module, and predictors therefore output categorical distributions over $d_h$ token identities.

A SLC model for an RHM of depth $L$ is comprised of $L-1$ SLC modules. Module $\ell$ reduces the learned level-$\ell$ sequence to an encoding of the level-$(\ell+1)$ latents. For a module at level $\ell$, the tensor shapes are
\[
\begin{array}{rcl}
\widehat h^{(\ell)}
    &\in& \mathbb R^{s^{L-\ell}\times d_h},\\[1mm]
x_u^{(\ell)}
    :=
    \big(\widehat h^{(\ell)}_{(u-1)s+1},\ldots,
    \widehat h^{(\ell)}_{us}\big)
    &\in& \mathbb R^{s\times d_h},\\[1mm]
\widehat\phi^{(\ell)}_u
    =
    \mathrm{Pred}^{(\ell)}(x_u^{(\ell)})
    &\in& \mathbb R^{s\times(s-1)\times s\times d_h},\\[1mm]
\mathrm{vec}\,\widehat\phi^{(\ell)}_u
    &\in& \mathbb R^{s^2(s-1)d_h},\\[1mm]
q_u^{(\ell+1)}
    =
    \mathrm{Clust}^{(\ell)}(\mathrm{vec}\,\widehat\phi^{(\ell)}_u)
    &\in& \Delta^{d_h-1},\\[1mm]
\widehat h^{(\ell+1)}
    &\in& \mathbb R^{s^{L-\ell-1}\times d_h}.
\end{array}
\]
Here $u=1,\ldots,s^{L-\ell-1}$ indexes level-$\ell$ patches, and $d_{h_2}$ is the hidden width of the CNN layers below. The dimensions of $\widehat\phi^{(\ell)}_u$ enumerate, respectively, the possible position of the input patch within its grandparent, the target cousin patch, the target token within that patch, and the $d_h$ token identity. This task is easiest for $d_h = v$, as clustering is architecturally enforced, and is hardest for $d_h\ge mv$ (when synonyms can be uniquely encoded without clustering). Unless otherwise noted, we will always use $d_h=mv$.

A SLC module is composed of a predictor and a clusterer. This arrangement is depicted in ~\Cref{fig:slc_schematic}. 

\begin{figure}
    \centering
    \begin{tikzpicture}[
    >={Stealth[length=2.2mm]},
    every node/.style={font=\small},
    patch/.style={ellipse, draw, minimum width=8mm, minimum height=11mm, inner sep=1pt},
    box/.style={draw, rectangle, rounded corners=1pt, inner sep=4pt, align=center},
    arr/.style={->, thick}
    ]

    \node (z1lab) at (0, 1.3) {$\widehat h^{(\ell)}_1$};
    \node (z2lab) at (0, 0.7) {$\widehat h^{(\ell)}_2$};
    \node (a) at (1.0, 1.3) {$a$};
    \node (b) at (1.0, 0.7) {$b$};
    \node[patch, fit=(a)(b), inner sep=0pt] (ab) {};

    \node (z3lab) at (0, -0.5) {$\widehat h^{(\ell)}_3$};
    \node (z4lab) at (0, -1.1) {$\widehat h^{(\ell)}_4$};
    \node (c) at (1.0, -0.5) {$c$};
    \node (d) at (1.0, -1.1) {$d$};
    \node[patch, fit=(c)(d), inner sep=0pt] (cd) {};

    \node at (1.0, -2.0) {$\vdots$};

    \node (pred) at (3.2, 1.0) {$\mathrm{Pred}^{(\ell)}$};

    \node (p3) at (5.3, 1.45) {$p(\widehat h^{(\ell)}_3\,|\,a,b)$};
    \node (p4) at (5.3, 0.55) {$p(\widehat h^{(\ell)}_4\,|\,a,b)$};

    \draw[decorate, decoration={brace, amplitude=4pt}]
        ($(p3.east) + (0.05,0.15)$) -- ($(p4.east) + (0.05,-0.15)$);
    \node (P1) at (6.8, 1.0) {$\widehat\phi^{(\ell)}_1$};

    \node[box] (cel) at (3.6, -0.9) {cross entropy loss};

    \node (clust) at (8.5, 1.0) {$\mathrm{Clust}^{(\ell)}$};

    \node (sm) at (10.5, 1.0) {$\SM$};

    \node (zout) at (12.1, 1.0) {$\widehat h^{(\ell+1)}_1$};

    \node[box, font=\footnotesize, align=center] (ccl) at (10.7, -0.9)
        {contrastive\\clustering loss};

    \node (Pi) at (P1 |- 0,-2.4) {$\widehat\phi^{(\ell)}_i$};
    \node at ($(Pi)+(-0.9,0)$) {$\cdots$};
    \node (zi) at (zout |- Pi) {$\widehat h^{(\ell+1)}_i$};

    \draw[arr] (ab.east) -- (ab.east -| pred.west);
    \draw[arr] (cd.east) -- (cd.east -| cel.west);

    \draw[arr] (pred.east) -- (pred.east -| p3.west);

    \draw[arr] (P1.east) -- (P1.east -| clust.west);

    \draw[arr] (sm.east) -- (sm.east -| zout.west);
    \draw[arr] (clust.east) -- (clust.east -| sm.west);

    \coordinate (junction) at (P1 |- cel.east);
    \draw[thick] (P1.south) -- (junction);
    \draw[arr] (junction) -- (cel.east);
    \draw[arr] (junction) -- (ccl.west);

    \coordinate (pijoin) at ($(ccl.south)+(-0.4,0)$);
    \draw[arr] (Pi.east) -- (pijoin |- Pi) -- (pijoin);

    \coordinate (zijoin) at ($(ccl.south)+(0.4,0)$);
    \draw[arr] (zi.west) -- (zijoin |- zi) -- (zijoin);

    \coordinate (z1join) at (ccl.east |- cel.east);
    \draw[thick] (zout.south) -- (zout |- z1join);
    \draw[arr] (zout |- z1join) -- (z1join);

    \coordinate (z1right) at ($(zout.east)+(0.6,0)$);
    \draw[arr] (zout.east) -- (zout.east -| z1right);

    \end{tikzpicture}

    \caption{\textbf{Schematic of a single SLC module, with the layer-wise prediction and clustering losses.}}
    \label{fig:slc_schematic}
\end{figure}
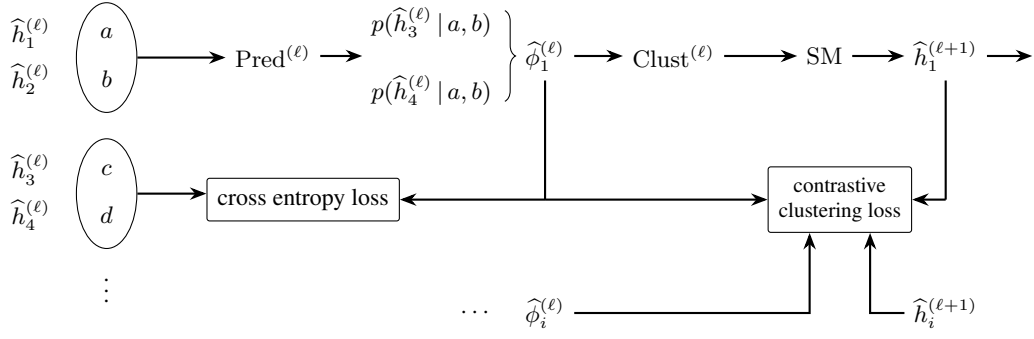

\begin{figure}
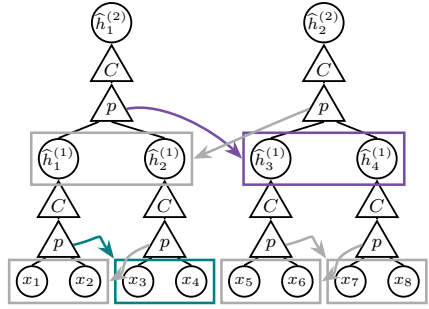

    \centering
    \resizebox{0.5\linewidth}{!}{\begingroup
\providecommand{\clusteringSchematicPanelSep}{}
\providecommand{\clusteringSchematicXSep}{}
\providecommand{\clusteringSchematicLevelOneY}{}
\providecommand{\clusteringSchematicLevelTwoY}{}
\providecommand{\clusteringSchematicRootY}{}
\renewcommand{\clusteringSchematicPanelSep}{8mm}
\renewcommand{\clusteringSchematicXSep}{0.80}
\renewcommand{\clusteringSchematicLevelOneY}{-2.3}
\providecommand{\clusteringSchematicPDeltaY}{1.3}
\providecommand{\clusteringSchematicCDeltaY}{0.7}
\renewcommand{\clusteringSchematicLevelTwoY}{-0.25}
\renewcommand{\clusteringSchematicRootY}{0.6}
\definecolor{ptcOchre}{RGB}{184,119,0}
\definecolor{ptcTeal}{RGB}{0,128,128}
\definecolor{ptcViolet}{RGB}{118,80,160}
\definecolor{ptcBlue}{RGB}{42,111,180}
\definecolor{ptcRed}{RGB}{190,72,72}
\definecolor{rhmTeal}{RGB}{0,128,128}
\definecolor{rhmViolet}{RGB}{118,80,160}
\providecommand{\clusteringSchematicXSep}{0.80}
\providecommand{\clusteringSchematicLevelOneY}{-1.20}
\providecommand{\clusteringSchematicLevelTwoY}{1.75}
\providecommand{\clusteringSchematicRootY}{2.12}
\begin{tikzpicture}[
    >=Stealth,
    node dot/.style={circle, draw, thick, fill=white, inner sep=0pt, minimum size=5.6mm},
    leaf dot/.style={circle, draw, thick, fill=white, inner sep=0pt, minimum size=4.6mm},
    module/.style={regular polygon, regular polygon sides=3, shape border rotate=0, draw, thick, fill=white, inner sep=0pt, minimum size=7.2mm, font=\footnotesize},
    target/.style={draw, thick, fill=white, minimum width=13mm, minimum height=6.4mm, inner sep=0pt},
    latent/.style={circle, draw, thick, fill=white, inner sep=0pt, minimum size=3.2mm},
    tree edge/.style={line width=0.8pt},
    signal/.style={->, line width=1.1pt}
]
    \path[use as bounding box] (-2,-4.55) rectangle (3.85,0.0);
    \IfFileExists{figs/clustering_schematic_ptc_module_panel_full_arrows.tikz}{\input{figs/clustering_schematic_ptc_module_panel_full_arrows.tikz}}{\input{clustering_schematic_ptc_module_panel_full_arrows.tikz}}
\end{tikzpicture}
\endgroup}
    \caption{\textbf{Multiple predictions improve signal}. We train the predictor and clusterer modules with weight sharing, and train cross-entropy prediction at all possible positions and cousin targets. \label{fig:slc_targetting}}
\end{figure}

\paragraph{Predictor.} The predictor operates on an $s$-tuple patch and predicts all $s(s-1)$ tokens with the same grandparent (cf. \Cref{fig:slc_targetting} for a visualization). Using all cousin targets increases the available signal and improves prefactors, although it does not change the asymptotic sample-complexity scaling. Because the same patch can occupy any of the $s$ positions under its grandparent and we use weight sharing across positions, the predictor outputs $s^2(s-1)$ distributions; for each observed patch, only the $s(s-1)$ distributions matching its actual position enter the loss.
The predictor is implemented in terms of convolutional layers, as 
\begin{equation}
    \widehat\phi^{(\ell)}(x) = \mathrm{Pred}^{(\ell)}(x) = (\SM\circ \CNN_3 \circ A\circ \CNN_2 \circ A \circ \CNN_1)(x)
\end{equation}
where $A = \relu \circ \mathrm{BN}$ denotes the activation function, $\CNN_1 : \mathbb{R}^{s\times d_h} \rightarrow \mathbb{R}^{d_{h_2}}$ is a stride $s$ 1d convolutional layer, while $\CNN_2  : \mathbb{R}^{ d_{h_2}} \rightarrow \mathbb{R}^{d_{h_2}}$ and $\CNN_3 : \mathbb{R}^{d_{h_2}}\rightarrow \mathbb{R}^{s\times (s-1)\times s \times d_h}$ are both stride 1. The softmax operation $\SM$ acts on the last dimension so that $\sum_{a=1}^{d_h} \widehat\phi^{(\ell)}(x)_{r\rho ta} =1 $. This last dimension is therefore the marginal probability distribution over neighbouring tokens, conditioned on the input patch position $r$, target cousin-patch index $\rho$, and target token position $t$.

We train the predictor with a cross-entropy prediction loss. Fix a grandparent index $g\in\{1,\ldots,s^{L-\ell-2}\}$ and an input patch position $r\in\{1,\ldots,s\}$ within that grandparent, so that $u=(g-1)s+r$ and the input patch is $x_u^{(\ell)}$. For a target cousin patch position $r'\ne r$, let $\rho_r(r')\in\{1,\ldots,s-1\}$ denote the index of $r'$ among the patch positions excluding $r$. The level-$\ell$ token at target position $t$ inside target patch $r'$ has absolute index
\[
    i(g,r',t):=((g-1)s+r'-1)s+t .
\]
For this input patch, the prediction loss is
\begin{equation}
    \mathcal{L}_\mathrm{pred}(g,r)
    =
    -\sum_{r'\ne r}\sum_{t=1}^s \sum_{a=1}^{d_h}
    \widehat h^{(\ell)}_{i(g,r',t),a}
    \log\!\left(
        \widehat\phi^{(\ell)}(x_u^{(\ell)})_{r,\rho_r(r'),t,a}
    \right).
\end{equation}
The predictor loss for module $\ell$ is the mean of this quantity over all grandparent indices $g$ and input patch positions $r$.

\paragraph{Clusterer.} The clusterer in turn assigns soft cluster labels to each patch's prediction vector -- this amounts to assigning a learned level-$(\ell+1)$ latent label.
\begin{equation}
    q^{(\ell+1)} = \mathrm{Clust}^{(\ell)}(\widehat\phi^{(\ell)}) = (\SM \circ \CNN_5 \circ A \circ \CNN_4 )(\widehat\phi^{(\ell)})
\end{equation}
where $\widehat\phi^{(\ell)}$ denotes the flattened output of $\mathrm{Pred}^{(\ell)}$, $\CNN_4 : \mathbb{R}^{s^2(s-1)d_h} \rightarrow \mathbb{R}^{d_{h_2}} $ and $\CNN_5 : \mathbb{R}^{d_{h_2}} \rightarrow \mathbb{R}^{d_h} $ are stride-1 CNNs. The soft code $q^{(\ell+1)}$ is the learned token $\widehat h^{(\ell+1)}$ passed to the next module. \textit{A priori} we do not specify that $d_h = v$, so the clusterer could choose to instead assign each synonym of a patch into its own cluster. Indeed, for finite data, differing but synonymous patches will imply different contexts due to finite sampling effects. We address this by using a contrastive clustering loss that penalizes assignments of sufficiently similar predictions into different clusters.

Let $q_i = \mathrm{Clust}^{(\ell)}(\widehat\phi^{(\ell)}_i)$ denote the soft cluster assignment for prediction vector $\widehat\phi^{(\ell)}_i$, where the index $i$ is taken to be over patch positions and batches, and we drop the $^{(\ell)}$ superscript for brevity. The predictions are then centred, so that $\widehat\phi^{(\ell)}_i \rightarrow \widehat\phi^{(\ell)}_i - \langle \widehat\phi^{(\ell)}_j  \rangle_j$. The similarity $S_{ij}$ between two predictions is the scaled cosine similarity
\begin{equation}
    S_{ij} = \frac{1}{2}\left(1+\frac{\widehat\phi^{(\ell)}_i\cdot\widehat\phi^{(\ell)}_j}{||\widehat\phi^{(\ell)}_i||_2||\widehat\phi^{(\ell)}_j||_2}\right)
\end{equation}
of these centred predictions.
The clustering loss is then
\begin{equation}
    \mathcal{L}_\mathrm{clust} = \langle \mathcal{L}_\mathrm{sim}(q_i,q_j,S_{ij}) + \lambda_{\mathrm{sep}}\mathcal{L}_\mathrm{sep}(q_i,q_j,S_{ij})\rangle_{ij} + \lambda_\mathrm{sparsity} \langle \mathcal{L}_\mathrm{sparsity}(q_i)\rangle_i
\end{equation}
where the $\langle \cdot\rangle_{ij}$ indicates an average over patch positions and batches. The similarity loss component
\begin{equation}
    \mathcal{L}_\mathrm{sim}(q_i,q_j,S_{ij}) = (q_i-q_j)\cdot(q_i-q_j)\relu(S_{ij} - S_{m})
\end{equation}
serves to penalize predictions that are sufficiently similar (more than the margin $S_m$) but assigned to different clusters (i.e. with non-zero $\Delta q$) -- this loss ensures that clusters consist of similar predictions. Alone, of course, the model could simply cluster ALL predictions into a single cluster. Therefore, we include a separation loss
\begin{equation}
    \mathcal{L}_{\mathrm{sep}}(q_i,q_j,S_{ij}) = q_i\cdot q_j(1-S_{ij})
\end{equation}
which penalizes dissimilar (i.e. large $1-S_{ij}$) predictions that are assigned into overlapping clusters (i.e. $q_i\cdot q_j > 0$). Finally, we introduce
\begin{equation}
    \mathcal{L}_\mathrm{sparsity}(q_i) = -q_i\cdot q_i
\end{equation}
which tends to maximize the $L_2$ norm of the cluster assignment and therefore the sparsity (since $||q_i||_1 = 1$ due to the softmax). Preliminary experiments demonstrated that this sparsity reward improves training stability and downstream interpretability.

A general caveat of employing contrastive losses is that memory costs can explode. Employed na\"ively, the first SLC module would compute $s^{2(L-1)}$ comparisons, each of which involves a potentially costly dot-product of a $s^2(s-1)d_h$ vector. In practice for each batch, we select a random subset of $N_\mathrm{compare}$ patch positions and batches on which to compute the clustering loss for memory constraints.

\paragraph{Training protocol and stabilizers.} To keep the prediction problem at depth $\ell$ comparable to the prediction problem at depth $0$, we tokenize clusterer outputs by means of a softmax operation, $\SM(X,T) = \exp(X/T)/\sum_i\exp(X_i/T)$. For $T=0$, this is hard label assignment. Preliminary experiments found that $T<1$ can improve convergence, but here we report results for $T=1$.

The clustering codebook dimension $d_h$ is another important stabilizer. If the space of possible labels is large, each input can be assigned a unique label; this is a failure to cluster, because synonyms of a latent are not assigned the same symbol in the next layer. The input vocabulary for layer $\ell+1$ then increases by a factor $m$, making the prediction task for layer $\ell+1$ intractable. This problem is exacerbated by noisy or finite sampling, where synonymous inputs do not produce identical predictions. This can be resolved by an architectural bottleneck, i.e. setting  $d_h = v$, as is done for the ILC. For most natural data, this dimension is unknown \textit{a priori}; to simulate these conditions we set the label dimension to $d_h>mv$ (the hardest case) and instead use the contrastive loss to set the scale of attraction between similar predictions.

To mitigate representation collapse, we use a teacher-student framework, in which latent targets are obtained from a teacher network whose weights are simply the exponentially lagged weights of the student. After each step of gradient descent, the teacher weights  $W^{(T)}$ are updated towards the student weights $W^{(S)}$ with  $W^{(T)} \leftarrow (1-\alpha_\mathrm{ema})W^{(T)} + \alpha_\mathrm{ema} W^{(S)}$, where $1/\alpha_\mathrm{ema}$ sets the exponential moving average timescale. The prediction loss therefore uses teacher representations to obtain targets $x_{:}$, which are evaluated against student predictions. The clustering loss meanwhile, computes the similarity $S_{ij}$ using the teacher's predictions and subsequently uses the student's clustering assignments $q_i$, $q_j$. As an alternative to the teacher-student framework, we highlight in \Cref{app:sec:slc_local_rules} that collapse can also be overcome by introducing stop-gradients between SLC modules, effectively making learning entirely local, reminiscent to predictive coding. Finally, we perform Jacobian descent, computing the gradient for each loss independently, and aggregate the gradients by means of the UPGrad algorithm. This selects an update in the dual-cone of the gradients to avoid conflicting updates~\cite{quintonJacobianDescentMultiobjective2025}, which we found in preliminary experiments reduced the need for careful hyperparameter selection.

Unless otherwise specified, SLC models are trained using AdamW and Jacobian descent (via TorchJD and UPGrad \citep{quintonJacobianDescentMultiobjective2025}) and following hyperparameters: $\alpha_\mathrm{ema} = 0.015$, $\mathrm{lr}=3\times10^{-3}$, $\mathrm{wd} = 10^{-3}$, $d_{h_2} = 150$, batch size 32, clustering dimension $d_h=128$,  $S_m = 0.8$, $\lambda_\mathrm{sep} = 0.5$, $N_\mathrm{compare}=300$, and $\lambda_\mathrm{spars} = 10^{-2}$. 
These were optimized by hyperparameter sweeps conducted at $L=5$ for $v=16$, $m=8$, $s=3$ using Optuna's Tree-Parzen-Estimator sampler and Hyperband pruner \cite{akibaOptunaNextgenerationHyperparameter2019}.

We use a validation set of 2000 samples to evaluate the model throughout training. When training classifiers on the final SLC representation, we use early stopping to select the best pretraining checkpoint and classification checkpoint based on this validation set performance, but report test set performance on a held-out test set of 2000 samples. In \Cref{fig:slc_L_scaling} the test sets are smaller, with 1000 samples.

\subsection{Training with local rules \label{app:sec:slc_local_rules}}
In \Cref{fig:slc_stop_grads}, we ablate the global training machinery while also testing increasingly biologically plausible training conditions. We first prevent gradient flow between SLC layers (denoted ``Layer SG'' for layerwise stop-gradient), then disable the exponential moving average (EMA) teacher network, and finally disable gradient flow between the prediction and clustering submodules. These changes remove the long-range backpropagating error signal and the biologically implausible time-delayed duplicate teacher network. When disabling EMA but allowing gradients to flow from the clustering loss through to the predictor, we observe representation collapse. We credit this to the clustering loss overpowering the prediction loss: trivial predictions are easy to cluster perfectly at the expense of prediction accuracy. Preventing gradient flow between prediction and clustering networks prevents this failure mode, showing that SLC can still learn when updates are constrained to local signals.

\begin{figure}
    \centering
    \includegraphics[width=0.5\linewidth]{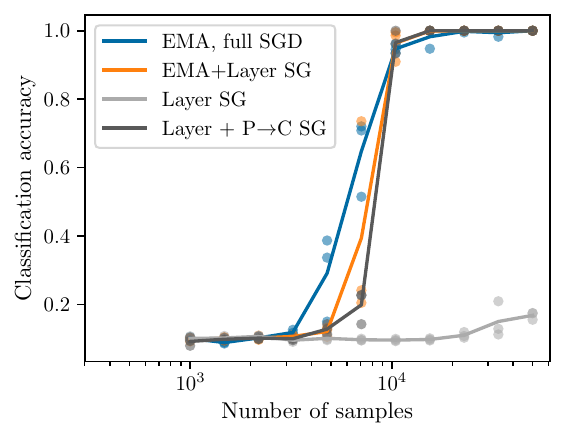}
    \caption{\textbf{Local learning suffices to learn the RHM.} We compare baseline training (blue) to increasingly local learning rules. ``Layer SG'' in legend indicates stopgradients between SLC layers, ``$P\rightarrow C$ SG'' indicates stopgradient between the predictor and clusterer submodules. EMA in legend indicates exponential-moving-average teacher is used as a target. RHM parameters are: $L=4, v=10, m=10, s=3$. Solid lines are average over three independent realizations.  }
    \label{fig:slc_stop_grads}
\end{figure}

\subsection{Training dynamics}
In \Cref{fig:slc_model_dynamics} we report the layer-wise accuracy and prediction loss of the different SLC modules. We evaluate the accuracy of module $\ell$ by taking the learned hard labels $\widehat h^{(\ell+1)}_i$ (with the $T=0$ hard label assignment) and comparing them with the ground-truth RHM latents $\latent{\ell+1}{i}$. We identify an optimal mapping between the assigned labellings and the true latents with the Kuhn-Munkres ``Hungarian'' algorithm, and report the accuracy attained with such a mapping.
\begin{figure}
    \centering
    \includegraphics[width=0.5\linewidth]{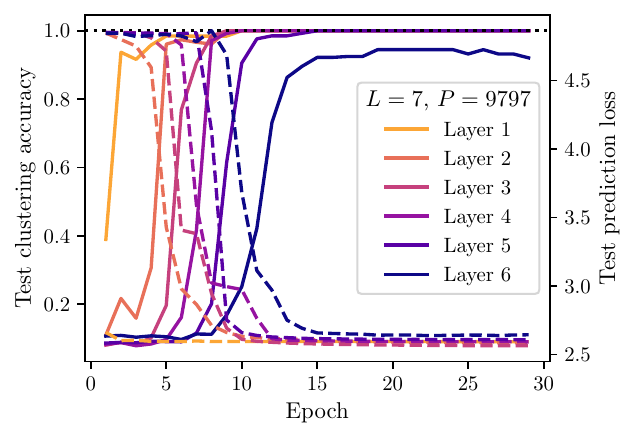}
    \caption{\textbf{Training proceeds layer-wise.} As each layer's clustering improves, the subsequent layer's prediction loss decreases, enabling that subsequent layer to then cluster.} \label{fig:slc_model_dynamics}
\end{figure}

\subsection{Independence of sample complexity with L}
In \Cref{fig:kmeans_and_slc_mL_plus_1} we check that the sample complexity scaling for ILC and SLC better supports $m^3$ rather than $m^{L+1}$. 
\begin{figure}
    \centering
    \includegraphics[width=0.95\linewidth]{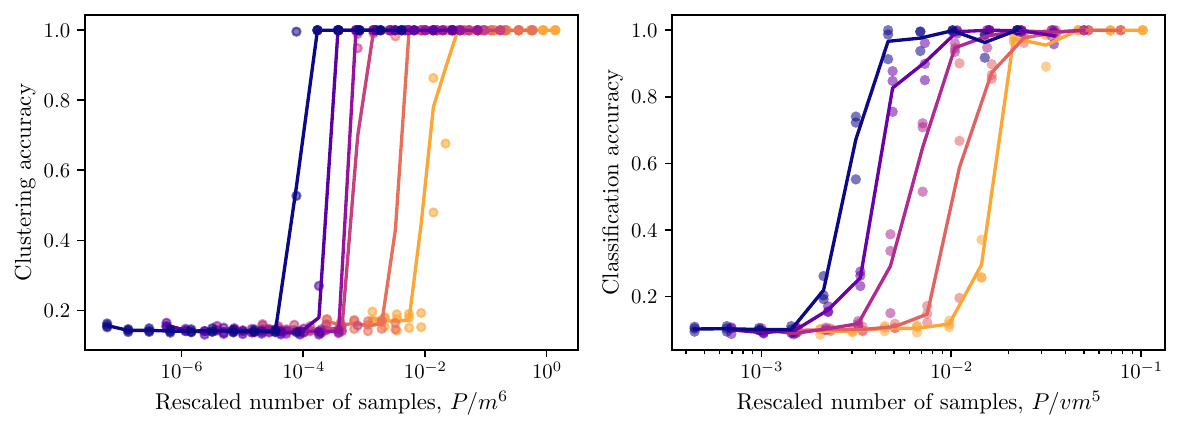}
    \caption{\textbf{Attempted curve collapse with token based SSL scaling.} Data are as in \Cref{fig:kmeans_and_slc}, but with the $m^{L+1}$ scaling that holds for token level SSL. The collapse of accuracy vs. training sample is significantly degraded. }
    \label{fig:kmeans_and_slc_mL_plus_1}
\end{figure}
Furthermore, in \Cref{fig:slc_L_scaling} we validate the sample complexity $P \sim m^3$ for systems with varying depth, up to $L=7$. 
\begin{figure}
    \centering
    \includegraphics[width=0.5\linewidth]{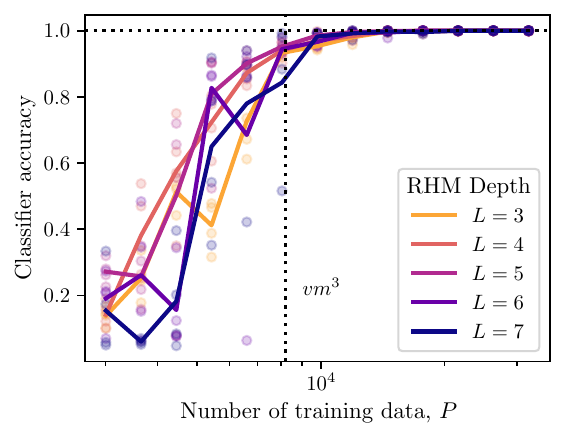}    \caption{\textbf{The SLC model exhibits a depth-independent sample complexity.} Accuracy of linear probes trained to classify the root latent $\latent{L}{1}$ from the final SLC representation.   RHM parameters are $s=3$, $m=8$, and $v=16$, with depth $L$ indicated by the legend. The pre-trained models were trained for 30 epochs, and the classifiers for the same duration and on the same dataset. Solid lines are an average over three independent RHM instantiations.}
    \label{fig:slc_L_scaling}
\end{figure}

\section{Further results on data2vec}
\label{app:data2vec}

All data2vec runs use the hyperparameters of \Cref{tab:hyperparams_data2vec_online_m6} unless stated otherwise.

\Cref{fig:app:d2vec_levels}--\textbf{left} shows the synonym clustering score for encoder layers to latents at different depths. At the root level $\ell=L=4$ the score stays $\approx 0$ in every layer, as expected since there is no signal from which the network can learn the root latent. For latents at level $\ell < L$, we find clustering throughout the middle layers of the network. This indicates that clustering occurs throughout the early layers of data2vec. Furthermore, the fact that for latent $\ell$ encoder layer $\ell-1$ is the first to exhibit clustering of the synonyms indicates that latents are constructed layer-by-layer.  Deeper layers exhibit stronger clustering of low-level latents, up to layer 2. This inversion is reminiscent of an encoder-decoder architecture, and mimics what is observed in U-Net architectures~\cite{favero2025compositional}. 

\Cref{fig:app:d2vec_levels}--\textbf{right} tests the presence of linear traces of latents in the teacher target for different $m$ and $\ell$. This data collapses with $vm^3$, confirming that the transition to all latents being linearly encoded in the teacher occurs at the predicted $vm^3$ scaling.

\paragraph{Offline setting.}
In \Cref{fig:d2vec_offline} we test the scaling collapse for data2vec instances trained in the offline setting, thereby decoupling the number of gradient descent steps from the number of original observed samples. We find $vm^3$ provides a good collapse. 

\paragraph{Squared loss.}
Our analysis idealizes the smooth-L1 loss of data2vec as a squared loss. With $\beta=4$, the smooth-L1 loss equals $r^2/(2\beta)=r^2/8$ for each per-coordinate residual $|r|<\beta$ and grows only linearly beyond. To check that the linear tails play no role, we repeat the online runs of \Cref{fig:d2vec_online_class} with the pure squared loss $\tfrac{1}{8}\sum_k r_k^2$, keeping everything else fixed. \Cref{fig:app:d2vec_controls}--\textbf{left} shows that the curves again collapse with $vm^3$ and that $P^*\propto m^3$, as with the smooth-L1 loss.

\paragraph{Token-level baseline.}
We train the same transformer on the same online RHM data ($v=16$, $s=2$, $L=4$, $m=3,\dots,7$, one seed) with a masked-language-modeling (MLM) objective: single tokens are masked independently with probability $0.15$, and a linear vocabulary head is trained with cross-entropy on the masked positions only. \Cref{fig:app:d2vec_controls}--\textbf{right} shows that the root-classification curves collapse when $P$ is rescaled by $vm^{L+1}=vm^5$, and that $P^*$ follows $m^{L+1}$ rather than $m^3$, as predicted for token-level objectives.

\begin{figure}
    \centering
    \includegraphics[width=1.05\linewidth]{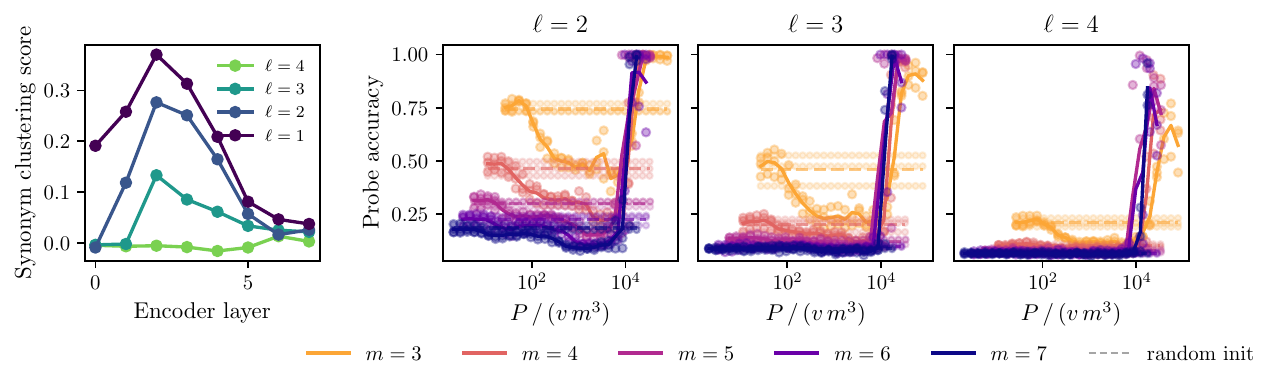}
    \vspace{-1em}
    \caption{
    \textbf{Left:} post-training synonym clustering score per encoder layer and RHM level $\ell$ ($m=6$). Values above $0$ indicate that synonymous level-$(\ell-1)$ tuples are clustered closer than non-synonymous ones. 
    \textbf{Right:} linear-probe accuracy for the latent $h^{(\ell)}$ from the teacher target, vs.\ rescaled sample size $P/(vm^3)$, for $\ell=2,3,4$. Curves collapse with $vm^3$ for every $\ell$.}
    \label{fig:app:d2vec_levels}\vspace{-1em}
\end{figure}

\begin{figure}
    \centering
    \includegraphics[width=.95\linewidth]{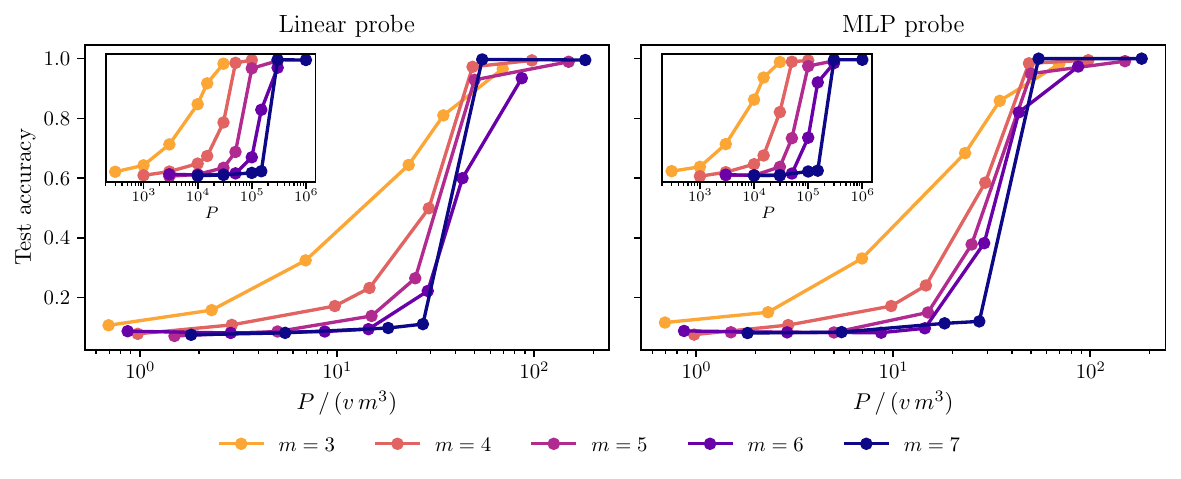}
    \caption{\textbf{Offline data2vec pretraining.} Linear and MLP probe test accuracy on the classification task as a function of the pretraining training-set size $P$. Each point is averaged over three seeds. Main axes: curves rescaled by $vm^3$. Insets: same data, raw $P$.}
    \label{fig:d2vec_offline}
\end{figure}

\begin{figure}[t]
    \centering
    \includegraphics[width=.49\linewidth]{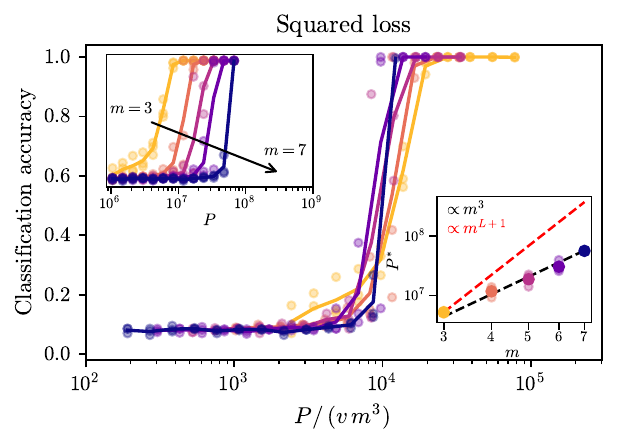}
    \includegraphics[width=.46\linewidth]{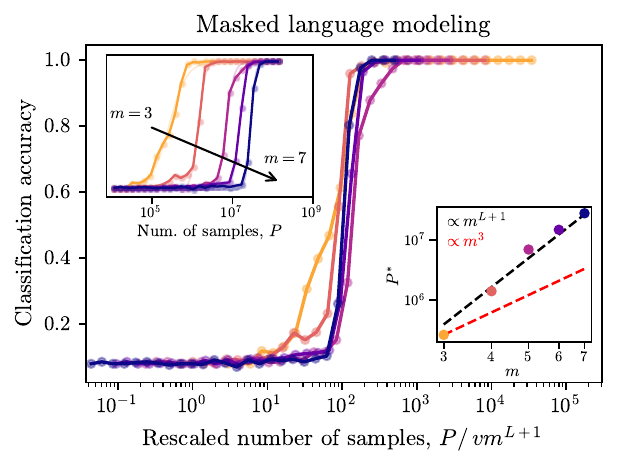}\hfill
    \caption{\textbf{MSE loss and MLM baseline.} Root-classification accuracy of probes on mean-pooled final-layer features (online, $v=16$, $s=2$, $L=4$).
    \textbf{Left:} data2vec with a pure squared loss in place of smooth-L1. Curves collapse with $vm^3$, as with the smooth-L1 loss. \textbf{Right:} the same transformer trained with masked-language modeling. Curves collapse with the token-level scaling $vm^{L+1}=vm^5$. Notice that in the online regime $P$ is tied to the number of gradient steps, so prefactors also reflect optimization and are not comparable across objectives. Insets: raw curves vs.\ $P$ (top left) and the $P^*$ at which accuracy reaches $50\%$, with $m^3$ and $m^{L+1}$ guides (bottom right).}
    \label{fig:app:d2vec_controls}
\end{figure}

\begin{table}[h]
\centering
\caption{Data2vec hyperparameters.}
\label{tab:hyperparams_data2vec_online_m6}
\renewcommand{\arraystretch}{1.05}
\begin{tabular}{@{}l l l@{}}
\toprule
\textbf{Group} & \textbf{Hyperparameter} & \textbf{Value} \\
\midrule
\multirow{8}{*}{Encoder architecture}
  & hidden dim $d_{\text{model}}$  & $2048$ \\
  & attention heads $n_h$          & $32$ \\
  & head dim                       & $64$ \\
  & encoder layers $n_l$           & $8$ \\
  & FFN dim $d_{\text{ff}}$        & $8192 = 4\,d_{\text{model}}$ \\
  & non-linearity                  & GELU \\
  & dropout                        & $0.1$ \\
  & positional embeddings          & learned, $T \times d_{\text{model}}$ \\
  & token embeddings               & learned, $V_{\text{tok}}\!=\!v\!+\!2$ ids \\
  & LayerNorm placement            & pre-LN \\
\midrule
\multirow{6}{*}{data2vec objective}
  & mask probability               & $0.15$ \\
  & mask span length               & $1$ \\
  & top-$K$ teacher layers averaged & $K = 4$ \\
  & per-layer LayerNorm of targets & yes \\
  & global LayerNorm of targets    & no \\
  & regression head depth          & $1$ linear layer \\
  & loss                           & smooth L1, $\beta = 4$ \\
  & teacher EMA decay $\mu$        & $0.99$ (constant; no annealing) \\
\midrule
\multirow{8}{*}{Optimization}
  & optimizer                      & AdamW \\
  & $(\beta_1, \beta_2)$           & $(0.9, 0.98)$ \\
  & weight decay                   & $0.01$ (excl. bias / LayerNorm) \\
  & learning rate                  & $\eta = 10^{-4}$ (constant) \\
  & LR schedule                    & none (no warmup, no decay) \\
  & gradient clipping              & global L2 norm $\leq 1.0$ \\
  & batch size $B$                 & $512$ \\
  & training steps $T_{\text{steps}}$ & $262\,144$ \\
\midrule
\multirow{6}{*}{Probes}
  & pooling                        & mean over positions \\
  & linear probe                   & \texttt{Linear}$(d_{\text{model}}, v)$ \\
  & MLP probe                      & $d_{\text{model}}\!\to\!2 d_{\text{model}}\!\to\!v$, GELU, dropout $0.1$ \\
  & probe optimizer                & Adam, lr $10^{-3}$ \\
  & probe steps per eval           & $2000$ (each: linear, MLP) \\
\bottomrule
\end{tabular}
\end{table}

\FloatBarrier
\section{Compute budget \label{app:sec:compute}}
SLC experiments were conducted on individual H100 nodes. The SLC experiments in \Cref{fig:kmeans_and_slc} took approximately 10 H100 hours. The $L$ scaling experiment reported in \Cref{fig:slc_L_scaling} cost approximately 100 H100 hours, as the amount of computations scales linearly with input data size which consists of $ s^L$ tokens (and therefore exponentially in depth). \Cref{fig:slc_stop_grads} cost approximately 3 H100 hours. We do not have an estimate for the Optuna hyperparameter sweeps or preliminary experiments. 

The data2vec experiments cost approximately 1,000 H100 hours. 
 
\section{Power-law fitting\label{app:statistical_fitting}}
To extract power-law fits of the form $P^\ast(m)\sim m^\alpha$, we perform monte-carlo sampling. For each value of $m$, we estimate the empirical average and standard-deviation of $P^\ast(m)$ from $N$ simulations. In each monte-carlo, we then draw $N$ iid samples from Gaussian distribution with that average and standard deviation, for each value of $m$ and perform a least-squares fit in log-log space. The distribution of fitted $\alpha$  after 300 Monte-Carlo iterations provides the $2\sigma$ uncertainty ranges reported in the captions of \Cref{fig:kmeans_and_slc,fig:d2vec_online_class}. The SLC fits are conducted over m=8-15, excluding m=7 as it has pre-asymptotic (small-m) factors that bias the fit; including that data results in a lower estimate $\alpha = 2.71\pm0.15$.

\end{document}